\documentclass[twoside]{article}
\usepackage{ecj,palatino,epsfig,latexsym,natbib}
\usepackage{amsmath, amsfonts, amsthm, amssymb}
\usepackage{algorithmic}
\usepackage[ruled,vlined,linesnumbered]{algorithm2e}
\usepackage{array}
\usepackage[caption=false,font=normalsize,labelfont=sf,textfont=sf]{subfig}
\usepackage{textcomp}
\usepackage{stfloats}
\usepackage{url}
\usepackage{verbatim}
\usepackage{graphicx}
\usepackage{makecell}
\usepackage{hyperref}
\usepackage{xr}
\usepackage{threeparttable}
\usepackage{framed} 

\usepackage{booktabs}

\usepackage{tikz,xcolor,hyperref}

\usepackage[switch, pagewise]{lineno}

\newtheorem{theorem}{Theorem}
\newtheorem{corollary}{Corollary}[theorem]
\newtheorem{lemma}{Lemma}
\newtheorem*{remark}{Remark}

\newtheorem{definition}{Definition}

\makeatletter
\newcommand*{\addFileDependency}[1]{
  \typeout{(#1)}
  \@addtofilelist{#1}
  \IfFileExists{#1}{}{\typeout{No file #1.}}
}
\makeatother

\newcommand*{\myexternaldocument}[1]{%
    \externaldocument{#1}%
    \addFileDependency{#1.tex}%
    \addFileDependency{#1.aux}%
}

\myexternaldocument{ecjsample-appendix}

\definecolor{lime}{HTML}{A6CE39}
\DeclareRobustCommand{\orcidicon}{
    \hspace{-2.5mm}
    \begin{tikzpicture}
    \draw[lime, fill=lime] (0,0)
    circle [radius=0.16]
    node[white] {{\fontfamily{qag}\selectfont \tiny ID}};
    \draw[white, fill=white] (-0.0625,0.095)
    circle [radius=0.007];
    \end{tikzpicture}
    \hspace{-2mm}
}

\foreach \x in {A, ..., Z}{%
    \expandafter\xdef\csname orcid\x\endcsname{\noexpand\href{https://orcid.org/\csname orcidauthor\x\endcsname}{\noexpand\orcidicon}}
}

\begin{document}
\ecjHeader{x}{x}{xxx-xxx}{2025}{Fitness Supremums on LGP}{Zhixing H., Yi M., Fangfang Z., Mengjie Z., Wolfgang B.}
\title{\bf Bridging Fitness With Search Spaces By Fitness Supremums: A Theoretical Study on LGP}  


\author{\name{\bf Zhixing Huang\orcidA{}}  \hfill 
\addr{zhixing.huang@ecs.vuw.ac.nz}\\ 
\name{\bf Yi Mei\orcidC{}}  \hfill 
\addr{yi.mei@ecs.vuw.ac.nz}\\ 
\name{\bf Fangfang Zhang\orcidB{}}  \hfill 
\addr{fangfang.zhang@ecs.vuw.ac.nz}\\ 
\name{\bf Mengjie Zhang\orcidD{}}  \hfill 
\addr{mengjie.zhang@ecs.vuw.ac.nz}\\ 
        \addr{Centre for Data Science and Artificial Intelligence \& School of Engineering and Computer Science, Victoria University of Wellington, Wellington, 6140, New Zealand}
\AND
       \name{\bf Wolfgang Banzhaf\orcidE{}} \hfill \addr{banzhafw@msu.edu}\\
        \addr{Department of Computer Science and Engineering, BEACON Center for the Study of Evolution in Action, and Ecology, Evolution and Behavior Program, Michigan State University, East Lansing, MI 48864, USA}
}

\maketitle

\begin{abstract}
Genetic programming has undergone rapid development in recent years. However, theoretical studies of genetic programming are far behind. One of the major obstacles to theoretical studies is the challenge of developing a model to describe the relationship between fitness values and program genotypes. In this paper, we take linear genetic programming (LGP) as an example to study the fitness-to-genotype relationship. 
{We find that the fitness expectation increases with fitness supremum over instruction editing distance, considering
1) the fitness supremum linearly increases with the instruction editing distance in LGP, 2) the fitness infimum is fixed, and 3) the fitness probabilities over different instruction editing distances are similar. }
We then extend these findings to explain the bloat effect and the minimum hitting time of LGP based on instruction editing distance.
The bloat effect happens because it is more likely to produce better offspring by adding instructions than by removing them given an instruction editing distance from the optimal program.
The analysis of the minimum hitting time suggests that for a basic LGP genetic operator (i.e., freemut), maintaining a necessarily small program size and mutating multiple instructions each time can improve LGP performance.
The reported empirical results verify our hypothesis.
\end{abstract}

\begin{keywords}
Genetic Programming, Fitness Supremum, Instruction Editing Distance, Bloat Effect, Minimum Hitting Time.
\end{keywords}

\section{Introduction}
Genetic programming (GP) is a representative of stochastic symbolic search methods. GP applies genetic operators to manipulate symbolic solutions based on given primitives and predefined program structures (e.g., tree structures). GP has been widely applied to knowledge discovery \cite{Huang2022SLGP, qu_automated_2023,al-helali_genetic_2024}, image classification \cite{fan_genetic_2024, sun_automatic_2024, haut_accelerating_2024}, and combinatorial optimization problems \cite{shi_novel_2022,huang_toward_2024,zeitrag_surrogate-assisted_2022}. 
Despite its success, most understanding of GP behavior is derived empirically and from experiments. For example, we usually initialize GP programs with small program sizes and encourage GP to increase program size during the search. However, there is no theoretical explanation for these recommended designs. We are also not sure why and when these recommended designs are effective.
It would be better to understand GP behaviors from a theoretical perspective first, before moving on to large-scale experiments.
One of the obstacles in theoretical studies of GP is the weak causal link between fitness values and search space configurations. A small move in the search space might lead to a huge difference in fitness values. This weak causal link renders modeling the relationship between fitness values and search spaces very hard.


In this paper, we analyze GP behaviors by fitness supremums (i.e., the possible worst case of fitness values). 
{Specifically, the expectation of GP fitness increases with the fitness supremum when:
\begin{enumerate}
    \item the fitness infimum is fixed (e.g., 0 in most cases of this paper as we focus on minimization problems),
    \item the probability of sampling programs with the same fitness value is similar when the fitness supremum increases.
\end{enumerate}
In this case, fitness expectation increases with fitness supremum as there is likely more larger fitness values. 
}

{With this in mind,} we prove that 1) the fitness supremum in GP has a simple relationship with genotype editing distance, and 2) the probability of sampling programs with the same fitness value over different genotype editing distances are similar. 
In other words, {the increment of fitness supremum defined by genotype editing distance implies the increment of fitness expectation.}
This provides a perspective for analyzing fitness values via genotype editing distance and allows to connect fitness values with search spaces. 
The empirical results show that the proposed model based on fitness supremums consistently explains the bloat effect and the choice of variation step size in GP.

Specifically, we take linear genetic programming (LGP) as an example to perform theoretical analyses.
LGP is one of the popular variants of GP methods \cite{Nordin1994, Brameier2007}. The most prominent features of LGP are its linear representation of register-based instructions and the sequential execution of these instructions. The linear representation of LGP allows us to model LGP individuals as vectors of instructions and to count the variation of LGP programs (e.g., the number of different instructions) straightforwardly. There have been a number of studies taking LGP as an example to analyze GP evolution. For example, Hu et al. \cite{Hu2011,Hu2012,Hu2013,hu_neutrality_2018} analyzed the robustness and evolvability in LGP.

Here, we focus on understanding LGP behavior based on the fitness supremum and instruction editing distance.  
There are three main contributions. 
First, we develop a linear model between the fitness supremum and the instruction editing distance in Section \ref{sec:fitnessgaptheo}, showing that a small instruction editing distance implies a small fitness supremum. 
Second, we analyze the distribution of fitness supremums over the entire search space based on the instruction editing distance in Section \ref{sec:d_overspace}. The distribution of the instruction editing distance theoretically backs up the strategy of searching from small to large programs and explains the bloat effect in LGP.
Third, we analyze the expected minimum hitting time of LGP search by the reduction of instruction editing distance based on a basic genetic operator in Section \ref{sec:individual_move}. Our model suggests that reasonably large variation step sizes are beneficial for LGP with the basic genetic operator, further verified by empirical results. 
In the rest of this paper, Section \ref{sec:preliminary} introduces the basic concepts of LGP, and Section \ref{sec:conclude} summarizes the conclusions. The proofs of lemmas are given in Appendix \ref{app:lemmaproof}.

\section{Preliminary}\label{sec:preliminary}

\subsection{Individual Representation}
\subsubsection{Linear Representation}
\begin{figure}
    \centering
    \includegraphics[scale=0.6, viewport=10 10 450 250, clip=true]{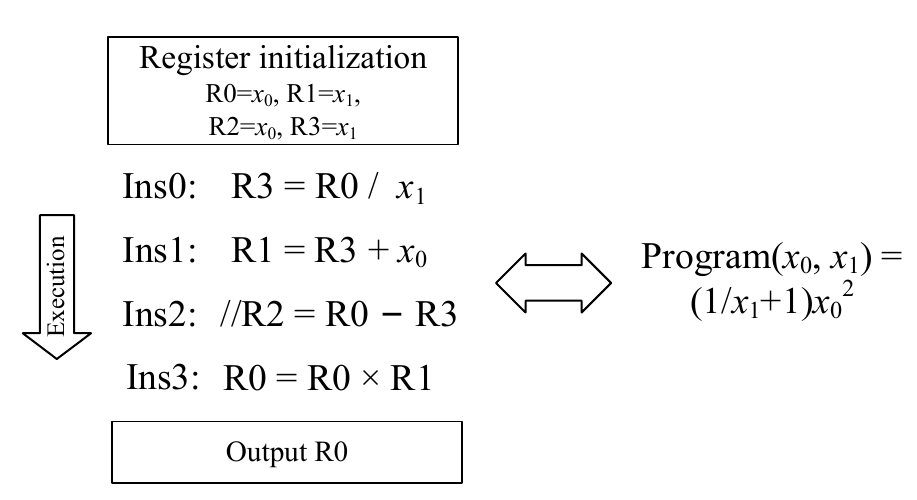}
    \caption{An LGP example composed of four instructions (from Ins0 to Ins3). Ins2 highlighted by ``//'' is an intron that does not affect the program output.}
    \label{fig:LGPexample}
\end{figure}

The representation of LGP individuals is a sequence of register-based instructions. Fig. \ref{fig:LGPexample} is an example of an LGP individual. Using the first instruction in Fig. \ref{fig:LGPexample} as an example, each instruction $\sigma$ consists of three parts: destination register ($\sigma_{des}=R3$), function ($\sigma_{fun}= / $), and source registers ($\sigma_{src}=\{R0, x_1\}$). The source registers specify the input of the function. The function accepts the values in source registers, executes, and stores the results in the destination register. For basic LGP using binary functions, each instruction has one destination register, one function, and two source registers. 
These registers and functions are given beforehand, and all the possible instructions form a combinatorial set of instructions $\mathcal{I}$.

An LGP program executes these instructions sequentially to form a computer program. Before execution of the instruction sequence, LGP performs pre-execution operations such as register initialization. The source register initialization assigns registers input features or constants. In Fig. \ref{fig:LGPexample}, LGP initializes registers by input features $x_0$ and $x_1$ alternatively, which is a common initialization strategy in LGP \cite{Brameier2007, Huang2022}. Then, LGP executes the instructions one by one. After that, LGP performs post-execution operations based on the designated output registers. Here, LGP outputs the value in $R0$ as the final output. Note that LGP can naturally have more than one output if we designate multiple output registers.

The linear representation of LGP has a number of advantages over the tree structures of basic GP. For example, the linear representation facilitates the reusing of intermediate results (i.e., building blocks) in LGP. Intermediate results are stored in registers. LGP easily reuses these results by taking them as source registers. In addition, the linear representation is very similar to many computer languages (e.g., C++ and Python), which is an efficient coding style. 

\subsubsection{Introns and Exons}
LGP individuals output the final results by output registers. However, not all the instructions in an LGP individual contribute to the final output. An instruction that does not contribute to the final output is called ``intron''. In contrast, an instruction that contributes to the final output is called ``exon''. In Fig. \ref{fig:LGPexample}, the third instruction commented with ``//'' is an intron, all other instructions are exons.

The possible number of introns increases with instruction positions \cite{Brameier2007}. 
Suppose an LGP program manipulates a register set of size $\gamma$, among which there are $\gamma_{\text{out}}$ output registers.
Then the possible number of introns grows from the first to the last instruction, from 0, $\frac{n}{\gamma}$, $\frac{2n}{\gamma}$, ..., to $\frac{(\gamma-\gamma_{\text{out}})n}{\gamma}$, where $n$ is the total number of instructions. Particularly, $\frac{n}{\gamma}$ is the number of possible introns when there are $\gamma - 1$ effective destination registers, and $\frac{(\gamma-\gamma_{\text{out}})n}{\gamma}$ is the number of possible introns when there are $\gamma_{\text{out}}$ effective destination registers. Effective destination registers are those registers contributing to the final output.
{Note that we do not know exactly which instruction is an intron unless a specific program is given, but mainly models the growth of the number of introns here.}
In Fig. \ref{fig:LGPexample}, the LGP program has four registers $R0$ to $R3$ in the primitive set and has one output register ($\gamma=4$, $\gamma_{\text{out}}=1$). Therefore, the number of possible introns at the position of the fourth instruction is $\frac{(\gamma-\gamma_{\text{out}})n}{\gamma}=\frac{3n}{4}$. At the positions of the second and third instruction, the number of possible introns is $\frac{n}{2}$ since there are two effective destination registers $R0$ and $R1$ given by the source registers of the final instruction.



\subsubsection{Four Levels of Search Information}
LGP has four levels of search information, i.e., genotype, phenotype, semantics, and fitness values. The genotype is the raw program representation that directly defines the search space of LGP search (i.e., instruction sequences). The phenotype is the effective part of the genotype (i.e., \{genotype\} $\backslash$ \{introns\}) or the observable representation of the genotype (e.g., a directed acyclic graph). The semantics is the behavior of an LGP program. The semantics of LGP programs is usually the output of an LGP program given a certain input. The fitness value is a judging metric over LGP program behavior. The fitness value defines the optimization objective in a straightforward manner.

The semantics of LGP programs is defined as a vector. Given $A$ different inputs, the semantics of an LGP program is defined as a $A\times (\gamma+B)$ dimensional vector, where $\gamma$ is the total number of registers and $B$ is the total number of input features. The input features are kept unchanged (e.g., $x_0$ and $x_1$ in Fig. \ref{fig:LGPexample}), while the register values change with program execution. Fig. \ref{fig:semantics_example} is an example of semantics over executing the LGP program of Fig. \ref{fig:LGPexample}. After initializing registers based on the given input, we have a 6-dimensional input semantic vector, where the first four elements are register values, and the last two elements are input features. The first instruction accepts the input semantics and outputs its results to the fourth register $R3$ (i.e., $2/3=0.67$). After all instructions are executed, the LGP program outputs the final semantics vector. There is a target semantics for approximation by the LGP system, although it may be unknown in some cases (e.g., unsupervised problems). Fig. \ref{fig:semantics_example} right shows a schematic diagram of the semantic movement, where the black dot is the given input semantics, the white dots are the output semantics of each instruction, and the black star is the target semantics. 

\begin{figure}
    \centering
    \includegraphics[scale=0.6, viewport=30 10 610 230, clip=true]{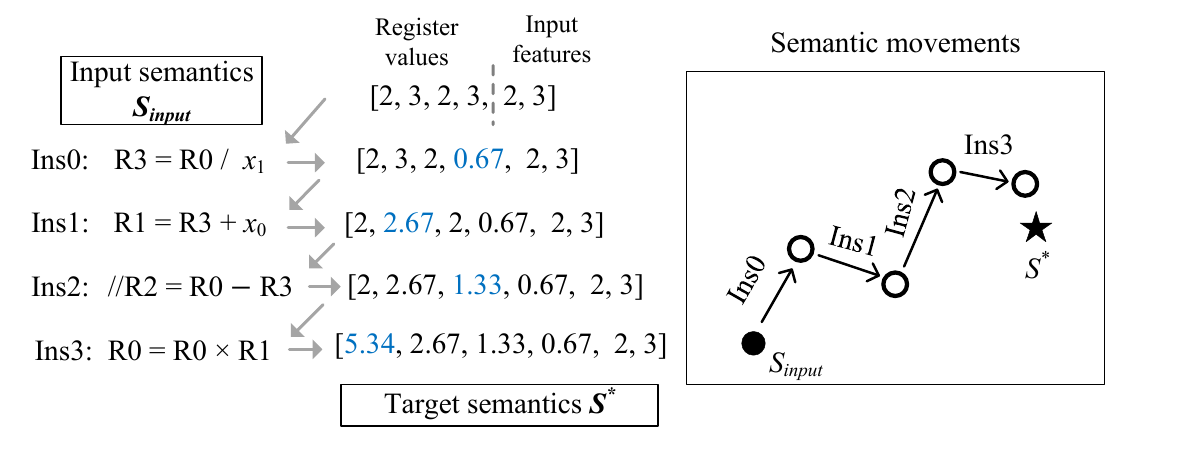}
    \caption{The semantics of executing the LGP program in Fig. \ref{fig:LGPexample} given an input of $[x_0, x_1]=[2,3]$. The manipulated registers are highlighted in blue font. The right figure is a schematic diagram of semantic movements by the example program in the 6-dimensional space.}
    \label{fig:semantics_example}
\end{figure}

\subsection{Evolutionary Framework}
LGP searches programs based on an evolutionary framework. First, LGP initializes a population of individuals, each individual representing a program. LGP evaluates these individuals by problem-specific fitness evaluation functions. As long as the stopping criteria are not satisfied, LGP iteratively selects parents, breeds new programs by applying genetic operators to the parents, and evaluates the fitness of these new programs. Finally, LGP outputs the best individual. In this paper, LGP optimizes minimization problems, that is, smaller fitness indicates better performance, and the minimum fitness is 0.

The genetic operators in LGP essentially define the neighborhood of an LGP individual. LGP individuals move within their neighborhood to explore the search space. There are three types of moves for LGP individuals, constructive moves, neutral moves, and destructive moves. 
A constructive move ($con$) means that an LGP individual moves closer to the optimal individual. A neutral move ($neu$) means that an individual remains at the same distance to the optimal individual after movement. A destructive move ($des$) means that an individual moves further away from the optimal one.

Many genetic operators for LGP have been explored. For the sake of simplicity, this paper takes adding and removing a random instruction as the only genetic operator of LGP (known as ``freemut'' in \cite{Brameier2007}). 
Freemut is one of the simplest operators for evolving LGP. It is commonly used in investigating LGP evolution \cite{hu_neutrality_2018}. Freemut has some advantages for investigating LGP evolution. For example, the number of changing instructions by freemut equals the instruction editing distance, simplifying analysis. 


\subsection{Related Work}
Existing theoretical studies of GP mainly focus on runtime analysis and the explanation of bloat effects. However, because of the weak causality between fitness values and search spaces, they have many over-strong assumptions and are far from practice.
For example, in runtime analysis, existing studies focus on two well-designed problems, ORDER and MAJORITY \cite{durrett_computational_2011,neumann_computational_2012,moraglio_runtime_2013,mambrini_analysis_2016}. ORDER and MAJORITY are two special problems that evaluate GP programs without explicit program execution, just using program inspection. Specifically, the correct programs in ORDER must apply a positive terminal $x_i$ before its complement $- x_i$. Correct programs in MAJORITY must apply all the positive terminals, and the number of positive terminals must be larger than the number of corresponding complements. By directly defining fitness based on program genotypes, ORDER and MAJORITY create strong causality between fitness values and search spaces, which facilitates theoretical analysis.
In addition, \cite{durrett_computational_2011,neumann_computational_2012,moraglio_runtime_2013,mambrini_analysis_2016} apply problem-specific genetic operators to model the variation on GP programs, which are uneasy to be extended to other domains.

Although some existing studies have explained the bloat effect of GP in a qualitative way or by empirical analysis \cite{soule_analysis_2002, brameier_neutral_2003, vanneschi_measuring_2010, sotto_studying_2016}, only a few studies model the bloat effect quantitatively. Freitag and Poli \cite{freitag_mcphee_schema_2001} explained the cause of bloat effect by comparing the expectation of program size over generations. However, the quantitative explanation of bloat effects in \cite{freitag_mcphee_schema_2001} is based on three assumptions. First, the GP population is infinitely large to cover the whole search space and only applies crossover to breed offspring. Second, the fitness function is a two-valued function, returning either $1$ or $1+\hat{f}$. Third, the expectation of program size is obtainable. Given that these assumptions hardly hold in practice, the theoretical explanation of \cite{freitag_mcphee_schema_2001} is restrictive.
Doerr et al. \cite{doerr_bounding_2018} also analyzed the running time of GP based on ORDER and MAJORITY, evolved by problem-specific genetic operators.

To the best of our knowledge, theoretical studies of GP are not kept up with the application of GP methods. They are restricted by using strong assumptions and tricky problem-specific designs. We believe that the weak causality between fitness values and search spaces is the key reason for these limitations. Here, we relax the weak causality between fitness values and search spaces by fitness supremums.

\subsection{Basic Definitions}
\begin{definition}\label{def:Psi}
    Given an input semantics $\mathbf{s}_{0}$ and a set of programs $\mathcal{P}$, the semantic space $\Psi(\mathbf{s}_{0} | \mathcal{P})$ is the set of output semantics of all the programs $\rho$ in $\mathcal{P}$ taking $\mathbf{s}_{0}$ as inputs.
    $$\Psi(\mathbf{s}_{0} | \mathcal{P}) = \{ \rho(\mathbf{s}_{0}) | \rho \in \mathcal{P} \}.$$
\end{definition}

{
\begin{remark}
    We simply assume the optimal program $\rho^*\in \mathcal{P}$, and the target output semantics $\mathbf{s}^*\in \Psi(\mathbf{s}_0|\mathcal{P})$.
\end{remark}

\begin{definition}
    Given a set of optimal programs $\mathcal{P}^*$, $\mathcal{I}^*$ is the instruction set of all the possible instructions in $\rho^*$.
    $$\mathcal{I}^*=\{\sigma|\sigma\in \rho^*, \rho^*\in\mathcal{P}^* \}.$$
\end{definition}

\begin{definition}
    Given a set of optimal programs $\mathcal{P}^*$, $\Psi^*$ is the semantic set that includes $\mathbf{s}_0$ and all the intermediate and output semantics after sequentially executing each instruction in $\rho^*\in \mathcal{P}^*$ given $\mathbf{s}_0$.
\end{definition}
}

\begin{definition}
\label{def:theta_s}
    Given a semantic space $\Psi$ and an optimal semantic space $\Psi^* \subseteq \Psi$, $\Delta_{\Psi}$ is the minimal upper bound on the distance between any semantics in $\Psi$ and any semantics in $\Psi^*$.
    $$\Delta_{\Psi} = \sup_{\mathbf{s}_1\in\Psi, \mathbf{s}_2\in \Psi^*} ||\mathbf{s}_1 - \mathbf{s}_2||.$$
\end{definition}

\begin{remark}
    If $\mathcal{P}$ contains the identity program $\rho_{\mathbb{I}}(\mathbf{s}) = \mathbf{s}$, then $ \Delta_{\Psi} \geq ||\mathbf{s}^* - \mathbf{s}_0||$
\end{remark}

\begin{definition}
\label{def:K}
    Given a semantic space $\Psi$ and a fitness function $f: \Psi \mapsto \mathbb{R}$, $\Delta_{f(\Psi)}$ is the minimal upper bound on the fitness difference between any one semantics in $\Psi$ and the target semantics $\mathbf{s}^*$, normalized by their semantic difference.
    $$\Delta_{f(\Psi)} = \sup_{\mathbf{s} \in \Psi} \frac{||f(\mathbf{s}) - f(\mathbf{s}^*)||}{||\mathbf{s} - \mathbf{s}^*||}.$$
\end{definition}

\begin{definition}
\label{def:theta_f}
    Given a semantic space $\Psi$ and the instruction set of optimal programs $\mathcal{I}^*$, 
    $\Delta_{(\mathcal{I}^*,\Psi)}$ is the minimal upper bound on the change in the difference between any two semantics in $\Psi$ by applying any single instruction in $\mathcal{I}^*$ to them.
    $$\Delta_{(\mathcal{I}^*,\Psi)} =\sup_{\substack{
        \mathbf{s}_1, \mathbf{s}_2 \in \Psi, \\
        \sigma\in \mathcal{I}^*}} (||\sigma(\mathbf{s}_1) - \sigma(\mathbf{s}_2)||-||\mathbf{s}_1 - \mathbf{s}_2||).$$
\end{definition}

\begin{definition}
\label{def:theta_F}
    Given a semantic space $\Psi$, a set of instructions $\mathcal{I}$ so that $\sigma(\mathbf{s}) \in \Psi$ $(\forall \mathbf{s} \in \Psi, \sigma \in \mathcal{I})$, and a set of instructions of optimal programs $\mathcal{I}^*\subseteq \mathcal{I}$,
    $\Delta_{(\mathcal{I}^2,\Psi)}$ is the minimal upper bound on the difference between the two differences: the difference between any two semantics in $\Psi$ and the difference after applying any instruction in $\mathcal{I}$ and any instruction in $\mathcal{I}^*$ to them, respectively.
    $$\Delta_{(\mathcal{I}^2,\Psi)} =\sup_{\substack{
        \mathbf{s}_1, \mathbf{s}_2 \in \Psi, \\
        \sigma_1\in \mathcal{I}, \sigma_2\in \mathcal{I}^*}} (||\sigma_1(\mathbf{s}_1) - \sigma_2(\mathbf{s}_2)||-||\mathbf{s}_1 - \mathbf{s}_2||).$$
\end{definition}

\begin{lemma}
\label{lemma:theta}
    Given a semantic space $\Psi$ and a set of instructions $\mathcal{I}$, we have $0 \leq \Delta_{(\mathcal{I}^*,\Psi)} \leq \Delta_{(\mathcal{I}^2,\Psi)}$. (For proof refer to Appendix \ref{prf:theta})
\end{lemma}



The notations used in this paper are summarized in Table \ref{tb:notations}. We will introduce their meanings when they first occur as well.
\begin{table}[t]
    \centering
    \caption{Notations and their meanings}
    \label{tb:notations}
    \scalebox{0.86}{
    \begin{tabular}{c|p{130mm}} \toprule
         $\overline{x}$ & The upper bound on $x$.  \\
         $\underline{x}$& The lower bound on $x$. \\
    $\mathbf{s}$& An LGP program semantics, represented as a vector of register values. \\
     $\mathbf{s}^*$& The target semantics. \\
     $\mathbf{s}_0$& The input semantics.\\
     $\sigma$& An instruction. $\sigma(\mathbf{s})$ is the output semantics by instruction $\sigma$ given $\mathbf{s}$.\\
     $\mathcal{I}$& The set of all possible instructions.\\
     $n$ & The size of $\mathcal{I}$, $|\mathcal{I}|=n$.\\
     $\Psi$& The semantic space given an instruction set $\mathcal{I}$ and problem inputs.\\
     $\rho$& A program, represented as a sequence of instructions, i.e., $\rho=\sigma_m \circ \sigma_{m-1} \circ \dots \circ \sigma_1(\cdot)$, where $m$ is the number of instructions (program size). \\
     $\mathcal{P}$& An LGP search space, which is essentially a set of programs.\\
     $\rho^*$& The optimal program.\\
     $|\rho|$& The number of instructions in program $\rho$. \\
     $\rho(\mathbf{s})$& The output semantics by program $\rho$, given input semantics $\mathbf{s}$.   \\
     $\delta_{\rho_1, \rho_2}$& The instruction editing distance between two programs $\rho_1$ and $\rho_2$. The instruction editing distance to an optimal program is $\delta^*(\rho)=\delta_{\rho^*, \rho}$.\\
     $f(x)$& The fitness value of $x$. $x$ can be a semantics $\mathbf{s}$ or a program (genotype) $\rho$. \\
     $m$& The program size. $m^*$ is the smallest program size that represents the optimal solution $\rho^*$. For a specific problem, there is a predefined maximum program size $L$.\\
     $\gamma$& The number of registers. $\gamma_{\text{out}}$ denotes the number of output registers.\\
     $\phi$& The program size of accessible optimal programs.\\
     $\Omega(m_1, m_2)$& The neutral bloating factor of an LGP program bloating from program size $m_1$ to $m_2$.\\
     $\Lambda(m_1,m_2)$& The non-neutral bloating factor of an LGP program bloating from program size $m_1$ to $m_2$.\\
     $\eta$& The normalization factor of duplicated programs after variation.\\
         \bottomrule
    \end{tabular} 
    }
\end{table}

\section{Smaller Instruction Editing Distance Implies Smaller Fitness Supremum}\label{sec:fitnessgaptheo}
This section first develops a linear relationship between instruction editing distance $\delta$ and fitness supremums. Second, we prove that the probability of fitness values over different instruction editing distance $\delta_{\rho_1,\rho_2}$ is similar when $\mathcal{I}$ and $f(\Psi)$ are well designed. The empirical results verify the relationship. 

\subsection{Fitness Supremums and Fitness Distributions}

\begin{theorem}[Supremums on Fitness Gaps]
    \label{theo:dtofit}
    Given an LGP search space 
    $$\mathcal{P}_{\mathcal{I}, L} := \{\sigma_L \circ \sigma_{L-1} \circ \dots \circ \sigma_1(\cdot)\ |\ \sigma_1, \dots, \sigma_L \in \mathcal{I} \},$$ 
    and its corresponding semantic space $\Psi$, where $\mathcal{I}$ is the set of instructions and $L$ is the maximal program length,
    a set of optimal LGP programs $\mathcal{P}^*_{\mathcal{I}^*,L}$
    and its corresponding semantic space $\Psi^*$ where $\mathcal{I}^*$ is the set of all the possible instructions in $\mathcal{P}^*$,
    an input semantics $\mathbf{s}_0$, 
    and a fitness function $f(\cdot)$, then the fitness gap between any LGP program $\rho \in \mathcal{P}_{\mathcal{I}, L}$ and any program $\rho^*\in\mathcal{P}^*_{\mathcal{I}^*,L}$ has the following supremums.
    \begin{equation}
        |f(\rho)-f(\rho^*)| \leq
        \Delta_{f(\Psi)} (\Delta_{(\mathcal{I}^2,\Psi)} \delta_{\rho,\rho^*} + \Delta_{(\mathcal{I}^*,\Psi)} (L-\delta_{\rho,\rho^*})), \label{eq:fitgap1}
    \end{equation}
    \begin{equation}
         |f(\rho)-f(\rho^*)| \leq       \Delta_{f(\Psi)} \Delta_{\Psi}, \label{eq:fitgap2}
    \end{equation}
    where $\delta_{\rho,\rho^*}$ stands for the number of different instructions between $\rho$ and $\rho^*$, i.e., their editing distance.
\end{theorem}

\begin{proof}
Denote $\rho$ as $\rho_1$ and $\rho^*$ and $\rho_2$,
let $\rho_1 = \sigma_{1,L}\circ \dots \circ\sigma_{1,1}(\cdot)$, and $\rho_2 = \sigma_{2,L}\circ\dots\circ\sigma_{2,1}(\cdot)$. Given the input semantics $\mathbf{s}_0$, let $\mathbf{s}_{ij}$ be the output semantics after the $j$th ($j = 1,\dots,L$) instruction of $\rho_i$ ($i = 1, 2$).

For the $j$th instruction, there are two possible situations:
\begin{itemize}
    \item [(a)] If $\sigma_{1,j} = \sigma_{2,j}$, then $||\mathbf{s}_{1j}-\mathbf{s}_{2j}|| \leq ||\mathbf{s}_{1(j-1)}-\mathbf{s}_{2(j-1)}|| + \Delta_{(\mathcal{I}^*, \Psi)}$ (from \textbf{Definition \ref{def:theta_f}}).
    \item [(b)] If $\sigma_{1,j} \neq \sigma_{2,j}$, then  $||\mathbf{s}_{1j}-\mathbf{s}_{2j}|| \leq ||\mathbf{s}_{1(j-1)}-\mathbf{s}_{2(j-1)}|| + \Delta_{(\mathcal{I}^2, \Psi)}$ (from \textbf{Definition \ref{def:theta_F}}).
\end{itemize}
$\rho_1$ and $\rho_2$ have $\delta_{\rho, \rho^*}$ different instructions and $L -\delta_{\rho, \rho^*}$ same instructions. Therefore, after enumerating the total $L$ instructions, the above situation (a) occurs $L -\delta_{\rho, \rho^*}$ times, and the above situation (b) occurs $\delta_{\rho, \rho^*}$ times. Therefore,
$$
||\rho(\mathbf{s}_0) - \rho^*(\mathbf{s}_0)|| \leq \Delta_{(\mathcal{I}^2, \Psi)}\delta_{\rho, \rho^*} + \Delta_{(\mathcal{I}^*, \Psi)}(L-\delta_{\rho, \rho^*}).
$$

Then, from \textbf{Definition \ref{def:K}}, we have (by substituting the supremums on $||\rho(\mathbf{s}_0) - \rho^*(\mathbf{s}_0)||$ as $||\mathbf{s} - \mathbf{s}^*||$)
\begin{equation}
     |f(\rho)-f(\rho^*)| \leq
        \Delta_{f(\Psi)} (\Delta_{(\mathcal{I}^2,\Psi)} \delta_{\rho,\rho^*} + \Delta_{(\mathcal{I}^*,\Psi)} (L-\delta_{\rho,\rho^*})),
\end{equation}
Eq. (\ref{eq:fitgap1}) is proven.

On the other hand, from \textbf{Definition} \ref{def:K} we have 
$$
|f(\rho)-f(\rho^*)| \leq \Delta_{f(\Psi)} ||\rho(\mathbf{s}_0) - \rho^*(\mathbf{s}_0)||.
$$
Then, from \textbf{Definition} \ref{def:theta_s} we have
$$
0 \leq ||\rho(\mathbf{s}_0) - \rho^*(\mathbf{s}_0)|| \leq \Delta_{\Psi}.
$$
Therefore, $|f(\rho)-f(\rho^*)| \leq       \Delta_{f(\Psi)} \Delta_{\Psi}$, Eq. (\ref{eq:fitgap2}) is proven.  
\end{proof}


\begin{remark}
Without loss of generality, we simply assume $\inf(f(\rho))=f(\rho^*)=0$ (e.g., the distance to the target semantics), then the fitness supremum of $\rho$ is re-written as: 
\begin{equation}
    \sup f(\rho)= \Delta_{f(\Psi)} \min\{(\Delta_{(\mathcal{I}^2,\Psi)}- \Delta_{(\mathcal{I}^*,\Psi)})\delta^*(\rho) + \Delta_{(\mathcal{I}^*,\Psi)} L, \Delta_{\Psi}\}. \label{eq:fitupper}
\end{equation}

    From \textbf{Lemma \ref{lemma:theta}}, we have $\Delta_{(\mathcal{I}^2,\Psi)}- \Delta_{(\mathcal{I}^*,\Psi)} \geq 0$. Therefore, \textbf{Theorem \ref{theo:dtofit}} (in particular Eq. (\ref{eq:fitgap1})) shows that within a predefined LGP search space and its corresponding semantic space, a smaller (larger) editing distance from the optimal program $\rho^*$ implies a smaller (larger) fitness supremum.
\end{remark}

\begin{theorem}[Similar Fitness Probability]\label{theo:similar_fit_dist}
Assuming an LGP search space $\mathcal{P}$ containing the target program ($\rho^* \in \mathcal{P}$) 
is
given, for any editing distance $d$, let the program subspace $\mathcal{P}^d = \{\rho \in \mathcal{P} | \delta_{\rho, \rho^*} \leq d\}$, and for any fitness $v \in f(\mathcal{P}^{d}) = \{f(\rho) | \rho \in \mathcal{P}^d\}$, let the program subspace $\mathcal{P}^{(d, v)} = \{\rho \in \mathcal{P}^d | f(\rho) = v\}$, then for any distances $d \geq d_0 > 0$ and $v \in f(\mathcal{P}^{d_0})$, if
\begin{equation}
\frac{||\mathcal{P}^{(d+1,v)}\setminus \mathcal{P}^{(d,v)}||}{||\mathcal{P}^{d+1}\setminus \mathcal{P}^{d}||} \leq \frac{||\mathcal{P}^{(d,v)}||}{||\mathcal{P}^{d}||} + \frac{||\mathcal{P}^{(d,v)}||}{||\mathcal{P}^{d+1}\setminus \mathcal{P}^{d}||}, \label{eq:small_dist_variation}
\end{equation}
then we have 
\begin{equation}
    \left | \frac{||\mathcal{P}^{(d+1,v)}||}{||\mathcal{P}^{d+1}||} - \frac{||\mathcal{P}^{(d,v)}||}{||\mathcal{P}^{d}||} \right | < \frac{||\mathcal{P}^{(d,v)}||}{||\mathcal{P}^{d}||}, \label{eq:similar_fit_dist}
\end{equation}
where $||\mathcal{P}||$ stands for the number of programs in the program space $\mathcal{P}$.
\end{theorem}
\begin{proof}
First, we have
    \begin{align*}
   & \frac{||\mathcal{P}^{(d+1,v)}||}{||\mathcal{P}^{d+1}||} - \frac{||\mathcal{P}^{(d,v)}||}{||\mathcal{P}^{d}||} \\
    = \;& \frac{||\mathcal{P}^{(d+1,v)}||\cdot ||\mathcal{P}^{d}|| - ||\mathcal{P}^{(d,v)}|| \cdot ||\mathcal{P}^{d+1}||}{||\mathcal{P}^{d+1}|| \cdot ||\mathcal{P}^{d}||} \\
    = \;& \frac{(||\mathcal{P}^{(d+1,v)}\setminus \mathcal{P}^{(d,v)}|| + || \mathcal{P}^{(d,v)} ||)\cdot ||\mathcal{P}^{d}|| - ||\mathcal{P}^{(d,v)}|| \cdot ||\mathcal{P}^{d+1}||}{||\mathcal{P}^{d+1}|| \cdot ||\mathcal{P}^{d}||} \\
    = \;& \frac{||\mathcal{P}^{(d+1,v)}\setminus \mathcal{P}^{(d,v)}||\cdot ||\mathcal{P}^{d}|| - ||\mathcal{P}^{(d,v)}|| \cdot ||\mathcal{P}^{d+1}\setminus \mathcal{P}^{d}||}{||\mathcal{P}^{d+1}|| \cdot ||\mathcal{P}^{d}||}.
\end{align*}
In addition, since
$$
\frac{||\mathcal{P}^{(d+1,v)}\setminus \mathcal{P}^{(d,v)}||}{||\mathcal{P}^{d+1}\setminus \mathcal{P}^{d}||} \leq \frac{||\mathcal{P}^{(d,v)}||}{||\mathcal{P}^{d}||} + \frac{||\mathcal{P}^{(d,v)}||}{||\mathcal{P}^{d+1}\setminus \mathcal{P}^{d}||},
$$
we have
\begin{align*}
   & ||\mathcal{P}^{(d+1,v)}\setminus \mathcal{P}^{(d,v)}||\cdot ||\mathcal{P}^{d}|| - ||\mathcal{P}^{(d,v)}|| \cdot ||\mathcal{P}^{d+1}\setminus \mathcal{P}^{d}|| \\
   \leq \; & ||\mathcal{P}^{(d,v)}|| \cdot ||\mathcal{P}^{d}||.
\end{align*}
Therefore, 
\begin{equation*}
    \frac{||\mathcal{P}^{(d+1,v)}||}{||\mathcal{P}^{d+1}||} - \frac{||\mathcal{P}^{(d,v)}||}{||\mathcal{P}^{d}||} \leq \frac{||\mathcal{P}^{(d,v)}||}{||\mathcal{P}^{d+1}||} < \frac{||\mathcal{P}^{(d,v)}||}{||\mathcal{P}^{d}||}. \label{eq:sim_fit_dist_pos}
\end{equation*}
Second, if 
$$
\frac{||\mathcal{P}^{(d+1,v)}||}{||\mathcal{P}^{d+1}||} - \frac{||\mathcal{P}^{(d,v)}||}{||\mathcal{P}^{d}||} < 0,
$$
Then
\begin{align*}
& \frac{||\mathcal{P}^{(d,v)}||}{||\mathcal{P}^{d}||} 
 - \frac{||\mathcal{P}^{(d+1,v)}||}{||\mathcal{P}^{d+1}||} \\
    = \;& \frac{||\mathcal{P}^{(d,v)}|| \cdot ||\mathcal{P}^{d+1}\setminus \mathcal{P}^{d}|| - ||\mathcal{P}^{(d+1,v)}\setminus \mathcal{P}^{(d,v)}||\cdot ||\mathcal{P}^{d}||}{||\mathcal{P}^{d+1}|| \cdot ||\mathcal{P}^{d}||} \\
    \leq \;& \frac{||\mathcal{P}^{(d,v)}|| \cdot ||\mathcal{P}^{d+1}\setminus \mathcal{P}^{d}||}{||\mathcal{P}^{d}||\cdot ||\mathcal{P}^{d+1}||} < \frac{||\mathcal{P}^{(d,v)}||}{||\mathcal{P}^{d}||}.
\end{align*}
Overall, we have 
$$
\left | \frac{||\mathcal{P}^{(d+1,v)}||}{||\mathcal{P}^{d+1}||} - \frac{||\mathcal{P}^{(d,v)}||}{||\mathcal{P}^{d}||} \right | < \frac{||\mathcal{P}^{(d,v)}||}{||\mathcal{P}^{d}||}.
$$
Note that $||\mathcal{P}^{d+1}||>||\mathcal{P}^{d}||>0$ as they can be easily constructed by adding instructions into $\rho^*$.
\end{proof}
\begin{remark}
    In \textbf{Theorem \ref{theo:similar_fit_dist}}, $\frac{||\mathcal{P}^{(d,v)}||}{||\mathcal{P}^{d}||}$ stands for the conditional probability 
    $$
    \frac{||\mathcal{P}^{(d,v)}||}{||\mathcal{P}^{d}||} =: \Pr(f(\rho) = v | \delta_{\rho,\rho^*} \leq d),
    $$ 
    and $\frac{||\mathcal{P}^{(d+1,v)}\setminus \mathcal{P}^{(d,v)}||}{||\mathcal{P}^{d+1}\setminus \mathcal{P}^{d}||}$ stands for the conditional probability 
    $$
    \frac{||\mathcal{P}^{(d+1,v)}\setminus \mathcal{P}^{(d,v)}||}{||\mathcal{P}^{d+1}\setminus \mathcal{P}^{d}||} =: \Pr(f(\rho) = v | \delta_{\rho,\rho^*} = d+1),
    $$
    where ``$=:$'' indicates denoting the left side by the right notation.
    Therefore, \textbf{Theorem \ref{theo:similar_fit_dist}} implies that when the LGP program space expands from distance $d$ to distance $d+1$ (by including more distant programs from the target program), if the probability for each fitness value does not increase by more than $\frac{||\mathcal{P}^{(d,v)}||}{||\mathcal{P}^{d+1}\setminus \mathcal{P}^{d}||}$ (Eq. (\ref{eq:small_dist_variation})), then the fitness probability over the search spaces $\mathcal{P}^{d+1}$ and $\mathcal{P}^{d}$ are also similar to each other (at most by $\Pr(f(\rho) = v | \delta_{\rho,\rho^*} \leq d)$ for the probability of each fitness value). 
    $\Pr(f(\rho) = v | \delta_{\rho,\rho^*} \leq d)$ is small enough when the fitness function has multiple possible values across the search space. For example, the R-square fitness function $R^2(\rho)\in [0,1]$ in regression results in small enough $\Pr(f(\rho) = v | \delta_{\rho,\rho^*} \leq d)$, while an indicator function of finding optimal programs or not $f(\rho)\in\{0,1\}$ does not.
    
\end{remark}

\textbf{Theorem \ref{theo:dtofit}} shows that reducing $\delta_{\rho,\rho^*}$ is equivalent to reducing the supremum on $f(\rho)$ based on a fixed infimum $f(\rho^*)=0$. 
\textbf{Theorem \ref{theo:similar_fit_dist}} shows that the fitness probability below the supremum is similar between the search spaces $\mathcal{P}^d$ and $\mathcal{P}^{d+1}$ when the fitness function satisfies Eq.~(\ref{eq:small_dist_variation}).
Therefore, the increase (reduction) of $\delta_{\rho,\rho^*}$ implies the increase (reduction) of the expectation of fitness over $\delta_{\rho,\rho^*}$,
which provides a connection between fitness and search space. 
This enables us to analyze the LGP evolution behaviors from the genotype perspective since LGP prefers moving to search spaces with better fitness, i.e., smaller $\delta_{\rho,\rho^*}$. 

\subsection{Verification of Fitness Supremum and Expectation}
\label{sec:estimateCoeff}


To verify the relationship between fitness supremum and its expectaction, we investigate the test performance of LGP with various fitness supremums, caused by different $\Delta_{(\mathcal{I}^*,\Psi)}$ and $\Delta_{(\mathcal{I}^2,\Psi)}$, and $L$. 
We select symbolic regression problems, a typical application of LGP, as our test problems.
We evolve LGP based on the recommended parameter settings in existing LGP studies \cite{Huang2022}. Specifically, LGP initializes programs with 5 to 20 random instructions. Each LGP program can use up to 100 instructions during evolution. The population size is 256 individuals, evolving for 200 generations. Both the adding and removing mutation in freemut have 45\% probability, with reproduction taking another 10\% probability. The tournament selection size is 7, and the elitism selection size is 3 (i.e., $\approx$1\% of the population size). LGP constructs its instruction set and programs based on a primitive set, $\{+,-,\times,\div,\sin,\cos,\sqrt{|\cdot|},\ln(|\cdot|),\text{problem inputs} \}$. The rest of the experiments in this paper follow these settings.

\begin{table}[]
    \centering
    \caption{Manipulated Instruction Set.}
    \label{tb:instructionset}
    \scalebox{0.9}{
    \begin{tabular}{@{}p{20mm}p{90mm}@{}} 
    \toprule
Name & \makecell[c]{Instruction Set}   \\ \midrule
default  &  $\mathcal{I}$  \\
fx1.1     &  $\mathcal{I}$ $\cup$ $\{\sigma'|\sigma'(x)=1.1\sigma(x),\forall \sigma(x)\in\mathcal{I}\}$ \\
fx2      & $\mathcal{I}$ $\cup$ $\{\sigma'|\sigma'(x)=2\sigma(x),\forall \sigma(x)\in\mathcal{I}\}$  \\
fx4      & $\mathcal{I}$ $\cup$ $\{\sigma'|\sigma'(x)=4\sigma(x),\forall \sigma(x)\in\mathcal{I}\}$ \\
exp      & $\mathcal{I}$ $\cup$ $\{\sigma'|\sigma'(x)=exp(x)\}$ \\
add+100  & $\mathcal{I}$ $\cup$ $\{\sigma'|\sigma'(x)=\sigma(x)+100,  \forall \sigma(x)\in\mathcal{I} \wedge \sigma_{fun}=``+''\}$ \\
add+1000 & $\mathcal{I}$ $\cup$ $\{\sigma'|\sigma'(x)=\sigma(x)+1000,  \forall \sigma(x)\in\mathcal{I} \wedge \sigma_{fun}=``+''\}$ \\
\bottomrule
\end{tabular}}
\end{table}


$\Delta_{(\mathcal{I}^*,\Psi)}$ and $\Delta_{(\mathcal{I}^2,\Psi)}$ are decided by the instruction set $\mathcal{I}$. When $\mathcal{I}$ includes instructions that behave very differently with slightly different inputs (e.g., $\exp(\cdot)$), $\Delta_{(\mathcal{I}^*,\Psi)}$ and $\Delta_{(\mathcal{I}^2,\Psi)}$ would be large. With this in mind, we design six instruction sets, as shown in Table \ref{tb:instructionset}. 
We ensure that the search spaces based on these instruction sets must include $\rho^*$. 

From ``fx1.1'' to ``fx4'', we include additional functions into the default primitive set by multiplying coefficients 1.1, 2, and 4, respectively. By scaling the output of each instruction, we scale $\Delta_{(\mathcal{I}^*,\Psi)}$ and $\Delta_{(\mathcal{I}^2,\Psi)}$ by a factor of $4$ in ``fx4'', $2$ in ``fx2'', and $1.1$ in ``fx1.1''. 
In ``exp'', ``add+100'', and ``add+1000'', we include one more function in the instruction set. Specifically, we include $\exp(x)$ in ``exp'', the instruction that adds 100 to the output for the instructions with addition in ``add+100'', and the instruction adding 1000 to the output for the instructions with addition in ``add+1000''. 
Given that the default instruction set $\mathcal{I}$ has been well designed to necessarily include $\rho^*$ by existing studies, it is hard to further reduce $\Delta_{(\mathcal{I}^*,\Psi)}$ and $\Delta_{(\mathcal{I}^2,\Psi)}$ from the default instruction set.

To verify GP performance, we test GP on ten benchmark problems, including Nguyen4, Nguyen5, Nguyen7, Keijzer11, R1, \href{https://archive.ics.uci.edu/dataset/291/airfoil+self+noise}{Airfoil}, \href{https://www.kaggle.com/datasets/schirmerchad/bostonhoustingmlnd}{BHouse}, Tower, 
\href{https://archive.ics.uci.edu/dataset/183/communities+and+crime}{CCN}, 
and \href{https://archive.ics.uci.edu/dataset/186/wine+quality}{Redwine}. 
LGP minimizes the relative square error (RSE, i.e., fitness) of programs on these problems.

Table \ref{tb:testperform_F} shows the test RSE of LGP with different instruction sets. The best mean performance is highlighted in bold. We perform the Friedman test and Wilcoxon rank-sum test with a significance level of 0.05 to analyze the results. 
Although ``fx1.1'', ``fx2'', and ``fx4'' all double the instruction set, indicating a much larger search space, ``fx1.1'' has statistically similar performance with $\mathcal{I}$, while the other two are significantly worse than $\mathcal{I}$ in many cases. 
In addition, methods with instructions whose outputs vary greatly with different inputs (e.g., ``fx4'' and ``exp'') are worse than those with instructions whose outputs are less sensitive, in terms of the mean rank.
For example, ``fx4'' (mean rank is $6.3$) is worse than ``fx2'' (mean rank is $5.5$), and ``exp'' (mean rank is $3.45$) is worse than ``add+100'' (mean rank is $3$). 
The results show that smaller $\Delta_{(\mathcal{I}^*,\Psi)}$ and $\Delta_{(\mathcal{I}^2,\Psi)}$ (i.e., a smaller upper bound on $\Delta f^*$) likely imply better test performance of LGP in practice. 

\begin{table*}[]
\caption{Test relative square error (standard deviation) of LGP with different instruction sets}
\label{tb:testperform_F}
\setlength{\tabcolsep}{3pt}
\scalebox{0.73}{
\begin{tabular}{cccccccc} 
\toprule
Problem   & default       & fx1.1              & fx2             & fx4             & exp             & add+100         & add+1000        \\ 
\midrule
Nguyen4   & 0.146 (0.103) & 0.188 (0.095) = & 0.221 (0.165) - & 0.237 (0.1) -   & \textbf{0.146 (0.101)} = & 0.161 (0.109) = & 0.188 (0.095) = \\
Nguyen5   & \textbf{0.081 (0.071)} & 0.095 (0.075) = & 0.084 (0.067) = & 0.284 (0.622) - & 0.121 (0.068) - & 0.112 (0.072) = & 0.095 (0.075) = \\
Nguyen7   & \textbf{0.001 (0)}     & 0.001 (0.002) = & 0.002 (0.005) - & 0.003 (0.006) - & 0.001 (0.001) = & 0.003 (0.015) = & 0.001 (0.002) = \\
Keijzer11 & 0.303 (0.143) & \textbf{0.256 (0.106)} = & 0.526 (0.414) - & 0.855 (1.834) - & 0.31 (0.156) =  & 0.308 (0.126) = & \textbf{0.256 (0.106)} = \\
R1        & 0.074 (0.033) & 0.073 (0.046) = & 0.079 (0.03) =  & 0.094 (0.057) - & \textbf{0.064 (0.04)} =  & 0.082 (0.046) = & 0.073 (0.046) = \\
Airfoil   & \textbf{0.601 (0.124)} & 0.612 (0.112) = & 0.69 (0.1) -    & 0.722 (0.132) - & 0.614 (0.113) = & 0.603 (0.121) = & 0.612 (0.112) = \\
Bhouse    & 0.432 (0.101) & 0.435 (0.105) = & 0.456 (0.109) = & 0.477 (0.106) - & 0.431 (0.112) = & \textbf{0.422 (0.127)} = & 0.435 (0.105) = \\
Tower     & 0.37 (0.048)  & 0.386 (0.062) = & 0.421 (0.08) -  & 0.419 (0.067) - & \textbf{0.369 (0.068)} = & 0.378 (0.06) =  & 0.386 (0.062) = \\
CNN       & 0.071 (0.081) & 0.055 (0.067) = & 0.079 (0.095) = & \textbf{0.053 (0.059)} = & 0.069 (0.108) = & 0.056 (0.072) = & 0.055 (0.067) = \\
Redwine   & 0.763 (0.034) & 0.771 (0.052) = & 0.8 (0.07) -    & 0.808 (0.065) - & \textbf{0.758 (0.037)} = & 0.773 (0.075) = & 0.771 (0.052) = \\ \hline
mean rank & \textbf{2.6}           & 3.95            & 5.5             & 6.3             & 3.45            & 3               & 3.2             \\
p-value   &               & 1               & 0.055           & 0.003           & 1               & 1               & 1        \\
\bottomrule
\end{tabular}
}
\end{table*}

\section{Estimation on Expectation of $\delta$ over Search Space}\label{sec:d_overspace}
This section estimates the expectation of $\delta$ over the search space by relaxing the expectation to its upper bounds. The expectation of $\delta$ allows us to estimate the expectation of fitness values over the search space. To facilitate the understanding of our modeling, we visualize the search space of LGP and name it the ``\textit{exploding lasagna model}''. 

\subsection{Exploding Lasagna Model}
\begin{figure}
    \centering
    \includegraphics[scale=0.7, viewport= 5 40 520 350, clip=true]{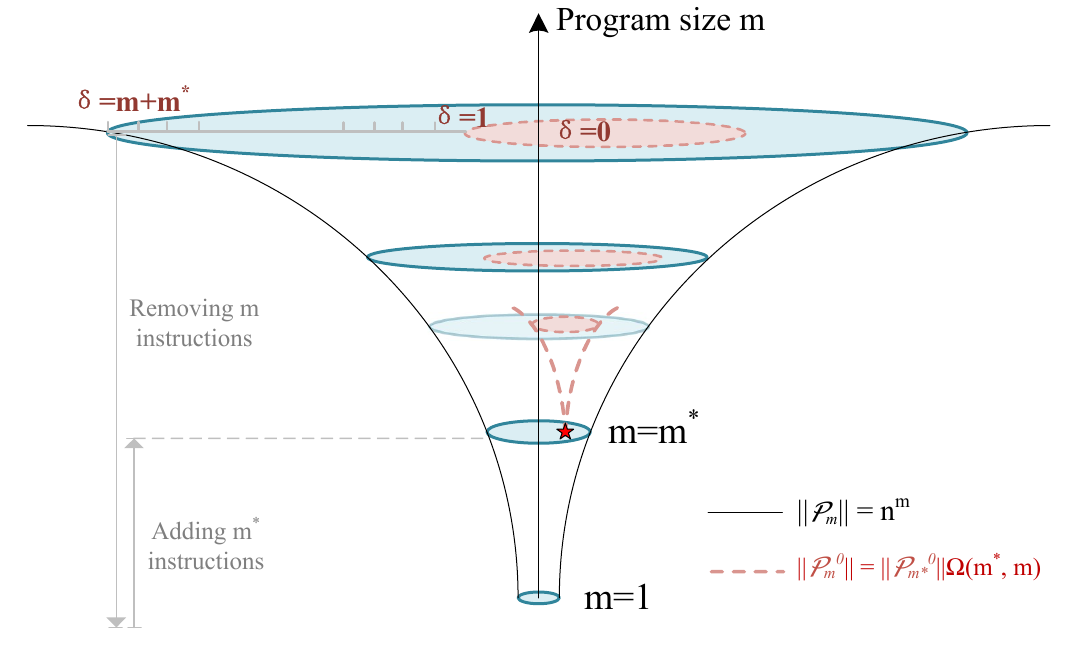}
    \caption{The exploding lasagna model for an LGP search space.}
    \label{fig:explodinglasagna}
\end{figure}
We divide the entire search space based on program size $m$.
Fig. \ref{fig:explodinglasagna} shows an exploding lasagna model. The term ``lasagna'' indicates that there are multiple layers in the models. Each layer in the model is composed of all the possible LGP solutions with the same program size $m$ (the same number of instructions). The term ``exploding'' indicates the exponentially increasing solution number in each layer. Given the total number of $n$ different instructions, there are $n^m$ possible programs in the $m$-th layer.

Let $\rho^{*}_{(m)}$ be an optimal/target program with $m$ instructions, and $m^*$ the size of the shortest possible optimal program, then from an optimal program $\rho^{*}_{(m)}$ ($m \geq  m^*$), we can obtain the optimal programs in its adjacent layers ($\rho^{*}_{(m\pm 1)}$) by adding or removing an intron instruction, as shown in the red regions in Fig. \ref{fig:explodinglasagna}. 
Specifically, the internal exploding lasagna model increases its program size with a bloating factor $\Omega(m^*, m) (m \geq m^*)$ (see \textbf{Lemma \ref{lemma:UBLBomega}}).  
An LGP solution jumps from a layer of $m$ to $m+1$ by adding an instruction and jumps to $m-1$ by removing an instruction. An LGP solution is at most $m^*+m$ away from $\rho^*$ (see \textbf{Lemma \ref{lemma:d}}).


\subsection{Estimation on $\delta$ Expectation over Search Space}
As $\delta^*(\rho)$ defines the fitness supremum and implies the increment of fitness expectation,
estimating $\delta^*(\rho)$ over the search space helps analyze the change of fitness expectation over the search space. We estimate the expectation of $\delta$ by counting the number of solutions that access $\rho^*$ with at least $\delta$ steps. 
Here, we focus on a case in which $m^*$ is included in the search space.



Let $|\rho|$ be the number of instructions in program $\rho$, $\delta^*(\rho)$ be the editing distance from a program $\rho$ to its closest optimal program, $\mathcal{P}_{m}$ be the set of programs with $m$ instructions (i.e., program size is $m$), $\mathcal{P}^{d}_{m}$ be the set of programs with size $m$ and $\delta^*(\rho)\leq d$, and $\mathcal{P}^{\delta^*=d}_{m}$ be the set of programs with size $m$ and $\delta^*(\rho)= d$, then the expected editing distance from a program with size $m$ to its closest target program is:
\begin{align}
\label{eq:Edsm}
\mathbb{E}[\delta^*(\rho)\ |\ \rho \in  \mathcal{P}_{m}] = \frac{\sum_{\rho \in \mathcal{P}_{m}} \delta^*_{\rho}}{n^m} = \sum_{d=\inf{\delta^*_{(m)}}}^{\sup{\delta^*_{(m)}}} \frac{d}{n^m}\cdot ||\mathcal{P}^{\delta^*=d}_m ||.
\end{align}

\begin{lemma}[Infimum and Supremum on $\delta^*$]
    \label{lemma:d}
     Suppose the smallest program size that represents $\rho^*$ is $m^*$, and LGP only uses adding or removing one random instruction as genetic operators to produce offspring. Then an LGP solution with size $m$ has infimum and supremum on $\delta^*_{(m)}$ to access $\rho^*$. Specifically,
    $$\inf{\delta^*_{(m)}}=\max\{0, m^*-m\} \leq \delta^*_{(m)} \leq m^*+m = \sup{\delta^*_{(m)}}.$$
    (For proof refer to Appendix \ref{prf:d})
\end{lemma}




\begin{lemma}[Lower and Upper Bounds on Neutral Bloating Factor $\Omega(m_1, m_2)$]
    \label{lemma:UBLBomega}
    The term ``neutral bloat'' indicates that the number of solutions with the same fitness increases when the program size increases by adding introns. 
    Suppose that an LGP program has $\gamma_{\text{out}}$ output registers ($\gamma_{\text{out}}\leq \gamma$). When the LGP program increases its size from $m_1$ to $m_2$, the increasing factor of the number of LGP programs by adding introns into this program is $\Omega(m_1, m_2)$. 
    Then the lower and upper bounds on the neutral bloating factor are:
    \begin{equation}
        \sum_{i=\omega}^{\gamma-\gamma_{\text{out}}} \left( \frac{in}{\gamma}\right)^{m_2-m_1} < \Omega(m_1, m_2) < 
        \left( \frac{m_2(\gamma-\gamma_{\text{out}}+\omega)(\gamma-\gamma_{\text{out}}-\omega+1)n}{2\gamma} \right)^{m_2 - m_1},
    \end{equation}
    where $\omega=\max\{0,\gamma-\gamma_{\text{out}}-m_1\}$.
    (For proof refer to Appendix \ref{prf:UBLBomega})
\end{lemma}

\begin{remark}
Lemma \ref{lemma:UBLBomega} shows that both the lower and upper bounds on $\Omega(m_1, m_2)$ are a function of $n^{m_2-m_1}$. We can imagine that the number of optimal solutions that are generated from $\mathcal{P}_{m^*}^0$ by adding introns is exponentially increasing with program size. 
Specifically, the ratio of $\mathcal{P}_{m}^0$ over $\mathcal{P}_{m}$ (i.e., $||\mathcal{P}_{m}^0|| / n^{m}$) for a given $m$ is at least a ratio of $\sum_{i=\omega}^{\gamma-\gamma_{\text{out}}} \left( \frac{i}{\gamma}\right)^{m-m_1}n^{-m_1}$ where $m_1 \geq m^*$.
This implies that finding the programs with the most concise exons (the rest of the program filled by introns, $m_1=m^*$) has a much higher probability than with redundant exons ($m_1>m^*$) in LGP.
For example, to search a target program $\rho^*(x)=x+1$, a concise program $\rho_1=\{\mathrm{R0}=x+1\}$ is much eaiser to be found than $\rho_2=\{\mathrm{R0}=0+1;\mathrm{R0}=\mathrm{R0}+x\}$.
This conclusion backs up the empirical observations in \cite{winkler_how_2024}.

Furthermore, because the introns in LGP, especially the structural introns, are easily detected by a backward visiting method \cite{Brameier2007}, LGP has a great potential to find concise solutions, which is essential for improving interpretability. 
\end{remark}

\begin{lemma}[Lower and Upper Bounds on Non-neutral Bloating Factor $\Lambda(m_1, m_2)$]
    \label{lemma:UBLBlambda}
    The lower and upper bounds on the $\Lambda(m_1, m_2)$ are:
    $\sum_{i=\gamma_{\text{out}}}^{\lambda} \left( \frac{in}{\gamma}\right)^{m_2-m_1}<\Lambda(m_1, m_2)<\left( \frac{m_2(\gamma_{\text{out}}+\lambda)(\lambda -\gamma_{\text{out}} +1)n}{2\gamma} \right)^{m_2 - m_1}$
    where $\lambda=\min\{\gamma,\gamma_{\text{out}}+m_1\}$.
     (For proof refer to Appendix \ref{prf:UBLBlambda})
\end{lemma}

\begin{theorem}[Similar $\delta^*$ Probability]
    \label{theo:similar_delta_m}
    The probability of $\delta^*=d$ given a program size  $m$ ($||\mathcal{P}^{\delta^*=d}_m||/n^m$) is similar over different $m$, i.e., $$| \frac{||\mathcal{P}^{\delta^*=d}_m||}{n^m} - \frac{||\mathcal{P}^{\delta^*=d}_{m+1}||}{n^{m+1}}|\leq \epsilon,$$ 
where $\epsilon$ is a small enough positive number.
\end{theorem}
\begin{proof}
    The program set $\mathcal{P}^{\delta^*=d}_{m+1}$ can be constructed from $\mathcal{P}^{\delta^*=d}_{m}$, $\mathcal{P}^{\delta^*=d-1}_{m}$, and $\mathcal{P}^{\delta^*=d+1}_{m}$ by adding one intron, or one wrong instruction, or one correct instruction, respectively. Then we have
    $$||\mathcal{P}^{\delta^*=d}_{m+1}|| = ||\mathcal{P}^{\delta^*=d}_{m}||\Omega(m,m+1) + ||\mathcal{P}^{\delta^*=d\pm1}_{m}||\Lambda(m,m+1).$$
    Therefore,
    \begin{equation}
    | \frac{||\mathcal{P}^{\delta^*=d}_m||}{n^m} - \frac{||\mathcal{P}^{\delta^*=d}_{m+1}||}{n^{m+1}}| 
    = \frac{1}{n^m}|||\mathcal{P}^{\delta^*=d}_{m}|| -\frac{||\mathcal{P}^{\delta^*=d}_{m}||\Omega(m,m+1) + ||\mathcal{P}^{\delta^*=d\pm1}_{m}||\Lambda(m,m+1)}{n}|. \nonumber
    \end{equation}
    Based on \textbf{Lemma \ref{lemma:UBLBomega}} and \textbf{Lemma \ref{lemma:UBLBlambda}}, both $\Omega(m,m+1)$ and $\Lambda(m,m+1)$ are linearly correlated to $n$ as both of their upper and lower bounds are linearly correlated to $n$. Therefore, we have
    $$| \frac{||\mathcal{P}^{\delta^*=d}_m||}{n^m} - \frac{||\mathcal{P}^{\delta^*=d}_{m+1}||}{n^{m+1}}|=\frac{1}{n^m}|\alpha_1||\mathcal{P}^{\delta^*=d}_{m}||+\alpha_2||\mathcal{P}^{\delta^*=d\pm1}_{m}|||,$$
    where $\alpha_1$ and $\alpha_2$ are coefficients of $\Omega(m,m+1)$ and $\Lambda(m,m+1)$ by reducing $n$. 
    Because $n$ is a combinatorial number of $\alpha_1$ and $\alpha_2$, and $\mathcal{P}^{\delta^*=\{d,d\pm1\}}_{m}$ is a subset of $\mathcal{P}_m$, 
    $|\alpha_1||\mathcal{P}^{\delta^*=d}_{m}||+\alpha_2||\mathcal{P}^{\delta^*=d\pm1}_{m}|||$ is much smaller than $n^m$ especially when $d$ is relatively small.
    Thus,
    $$| \frac{||\mathcal{P}^{\delta^*=d}_m||}{n^m} - \frac{||\mathcal{P}^{\delta^*=d}_{m+1}||}{n^{m+1}}|\leq \epsilon.$$
\end{proof}

\begin{theorem}
    \label{theo:UB_Emin}
    The expectation of $\delta^*(\rho)$ increases with program size $m$, i.e., 
    $$\mathbb{E}[\delta^*(\rho)\ |\ \rho \in  \mathcal{P}_{m+1}] - \mathbb{E}[\delta^*(\rho)\ |\ \rho \in  \mathcal{P}_{m}]>0$$
\end{theorem}
\begin{proof}
    \begin{equation}
     \mathbb{E}[\delta^*(\rho)\ |\ \rho \in  \mathcal{P}_{m+1}] - \mathbb{E}[\delta^*(\rho)\ |\ \rho \in  \mathcal{P}_{m}] 
     = \sum_{d=\inf{\delta^*_{(m+1)}}}^{\sup{\delta^*_{(m+1)}}} \frac{d}{n^{m+1}}\cdot ||\mathcal{P}^{\delta^*=d}_{m+1} || - \sum_{d=\inf{\delta^*_{(m)}}}^{\sup{\delta^*_{(m)}}} \frac{d}{n^m}\cdot ||\mathcal{P}^{\delta^*=d}_m ||
    \end{equation}

    Based on \textbf{Lemma \ref{lemma:d}}, $\inf{\delta^*_{(m+1)}}\leq \inf{\delta^*_{(m)}}$, and $\sup{\delta^*_{(m+1)}}> \sup{\delta^*_{(m)}}$. Because $\frac{||\mathcal{P}^{\delta^*=d}_m||}{n^m}$ and $ \frac{||\mathcal{P}^{\delta^*=d}_{m+1}||}{n^{m+1}}$ are similar enough and there are more possible $d$ in $\mathbb{E}[\delta^*(\rho)\ |\ \rho \in  \mathcal{P}_{m+1}]$,
    $$\mathbb{E}[\delta^*(\rho)\ |\ \rho \in  \mathcal{P}_{m+1}] - \mathbb{E}[\delta^*(\rho)\ |\ \rho \in  \mathcal{P}_{m}]>0.$$
\end{proof}

Based on Theorem \ref{theo:dtofit} to \ref{theo:UB_Emin}, we analyze the change of fitness expectation over the search space $\mathbb{E}[f(\rho(m)]$.
Let $\mathbb{E}[\delta^*(\rho_{(m)})]$ be the abbreviation of $\mathbb{E}[\delta^*(\rho)\ |\ \rho\in\mathcal{P}_m]$, we have the following chain.
\begin{framed}
\centering
\setlength{\arraycolsep}{1pt}
    $\begin{array}{ccccc}
 \mathbb{E}[\delta^*(\rho_{(m)})] & \rightarrow & \mathbb{E}[ \sup{f(\rho_{(m)})}] & \rightarrow & \mathbb{E}[f(\rho_{(m)})] 
\end{array}$
\end{framed}
Specifically, the expectation of $\delta^*(\rho)$ increases with program size $m$, $\delta^*(\rho)$ defines the fitness supremum on $\rho$ with size $m$ ($ \sup{f(\rho_{(m)})}$) based on Theorem \ref{theo:dtofit}. The increase of the expectation of fitness supremum implies the increase of the fitness expectation $\mathbb{E}[f(\rho_{(m)})]$ as the fitness infimum is 0 and $f(\rho_{(m)})$ distribution is similar over $\delta^*$ (Theorem \ref{theo:similar_fit_dist}). Therefore, search space with larger program size has larger (worse) fitness expectation.

\subsection{Empirical Verification}\label{sec:exp_Edm}
\textbf{Theorem \ref{theo:UB_Emin}} gives us a way to estimate the expectation of $\delta^*$ over the search space with a program size of $m$ (i.e., $\mathbb{E}[\delta^*(\rho)\ |\ \rho\in\mathcal{P}_m]$). 
To verify $\mathbb{E}[\delta^*(\rho)\ |\ \rho\in\mathcal{P}_m]$, we investigate the mean fitness of randomly sampled programs with different program sizes over 50 independent runs, as shown in Fig. \ref{fig:Emind_v_meanfit}. 
We can see that the mean fitness (RSE) of randomly sampled programs increases with the program size $m$ in the four example symbolic regression problems (i.e., Nguyen4, Nguyen5, Keijzer11, and R1), verifying the increase of $\mathbb{E}[\delta^*(\rho)\ |\ \rho\in\mathcal{P}_m]$ and $\mathbb{E}[f(\rho_{(m)})]$.


\begin{figure}[t]
    \centering
    \includegraphics[scale=0.55, viewport=15 25 590 480, clip=true]{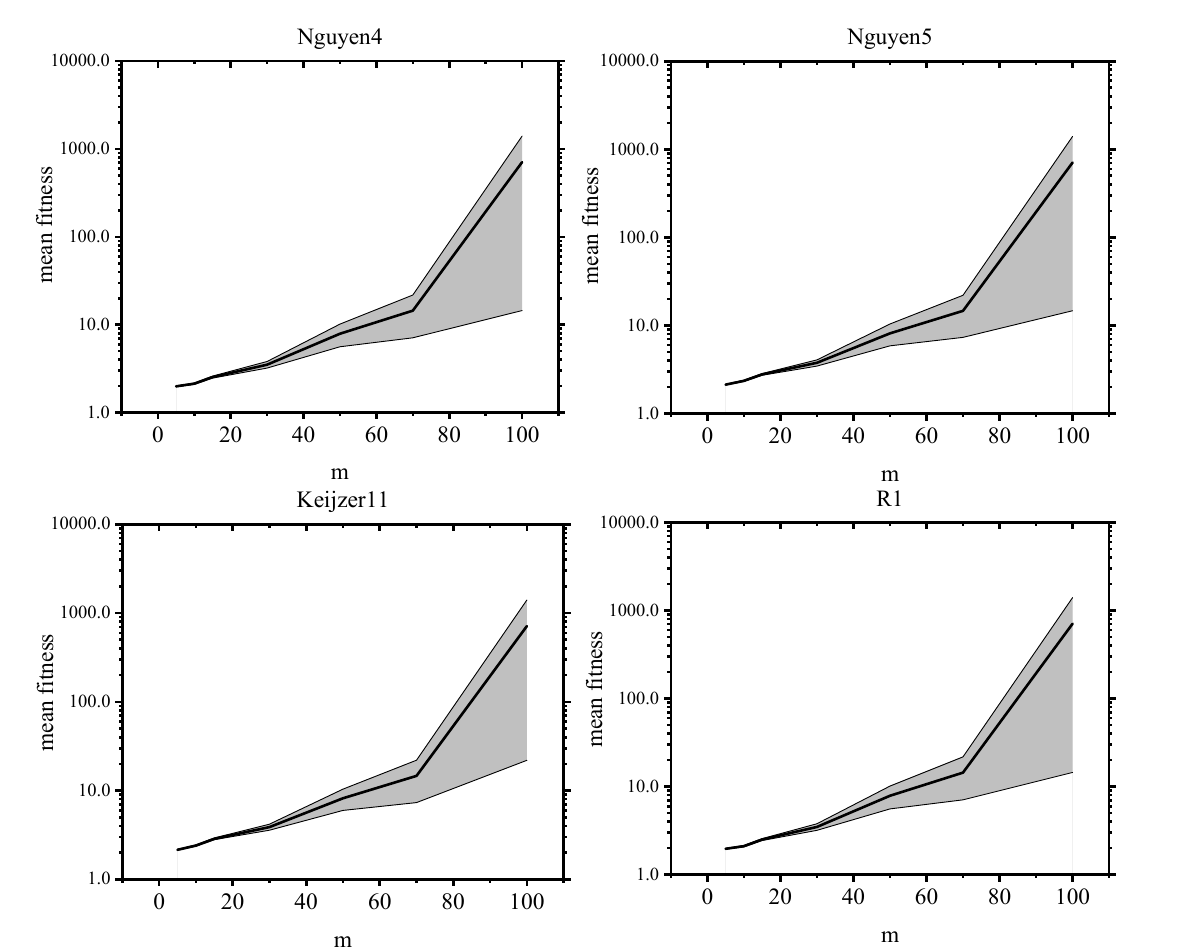}
    \caption{The mean fitness (RSE) over LGP initial populations with program size $m$. The shadow indicates the standard deviation of the fitness over 50 independent runs for a given program size. }
    \label{fig:Emind_v_meanfit}
\end{figure}

\subsection{Bloat Effect in LGP}
The bloat effect is the phenomenon that GP program size increases significantly while fitness changes slightly or not at all.
This section models the LGP bloat effect by \textbf{Definition \ref{def:constructMove}}.
Let $\mathbb{E}[d_m - \delta^*(\rho)\ |\ \rho\in\mathcal{P}^{d_m}_{m + \Delta m}]$ be the expectation of $\delta^*$ reduction and invariance when we sample programs from the search space with a program size of $m+\Delta m$, given a certain program size $m$ and an editing distance $d_m$. The expectation of $\delta^*$ reduction and invariance implies the fitness reduction and invariance after adding or removing $\Delta m$ instructions. Given that LGP prefers moving to search spaces with small or the same fitness (in minimizing problems), we would analyze the variation of program size $\Delta m$ based on the expectation of $\delta^*$.
\begin{definition}
     \label{def:constructMove}
    Given a program size $m$ and an editing distance $d_m$, the expectation of $\delta^*$ reduction and invariance after adding or removing instructions is defined as:
    \begin{equation}
        \mathbb{E}[d_m - \delta^*(\rho)\ |\ \rho\in\mathcal{P}^{d_m}_{m + \Delta m}] 
        =\left(\sum_{d=\inf{\delta^*_{(m+\Delta m)}}}^{\min\{\sup{\delta^*_{(m+\Delta m)}},d_m\}} \frac{||\mathcal{P}^{\delta^*=d}_{m+\Delta m}||}{n^{m+\Delta m}} (d_m-d) \right). 
    \end{equation}
\end{definition}

\begin{corollary}
\label{corol:bloat_effect}
    Given a program size $m$ and an editing distance $d_m$, the expectation of editing distance reduction of newly sampled program $\rho$ in the program set with size of $m+1$ is larger than in the program set with size of $m$, i.e.,   
    $$\mathbb{E}[d_m - \delta^*(\rho)\ |\ \rho\in\mathcal{P}^{d_m}_{m + 1}] - \mathbb{E}[d_m - \delta^*(\rho)\ |\ \rho\in\mathcal{P}^{d_m}_{m }]\geq0.$$
\end{corollary}
\begin{proof}
    Based on \textbf{Lemma \ref{lemma:d}}, we have $\inf{\delta^*_{(m+1)}} \leq \inf{\delta^*_{(m)}}$ and $\min\{\sup{\delta^*_{(m+1)}},d_m\} \geq \min\{\sup{\delta^*_{(m)}},d_m\}$. In other words, there are more adding items in $\mathbb{E}[d_m - \delta^*(\rho)\ |\ \rho\in\mathcal{P}^{d_m}_{m + 1}]$ than in $\mathbb{E}[d_m - \delta^*(\rho)\ |\ \rho\in\mathcal{P}^{d_m}_{m }]$.
    Then based on \textbf{Theorem \ref{theo:similar_delta_m}}, we have
    $$\mathbb{E}[d_m - \delta^*(\rho)\ |\ \rho\in\mathcal{P}^{d_m}_{m + 1}] - \mathbb{E}[d_m - \delta^*(\rho)\ |\ \rho\in\mathcal{P}^{d_m}_{m }]\geq0.$$
\end{proof}

\begin{remark}
    \textbf{Corollary \ref{corol:bloat_effect}} implies 
    $$\mathbb{E}[d_m - \delta^*(\rho)\ |\ \rho\in\mathcal{P}^{d_m}_{m + 1}] - \mathbb{E}[d_m - \delta^*(\rho)\ |\ \rho\in\mathcal{P}^{d_m}_{m-1}]\geq0.$$
    Adding instructions is more likely to reduce and maintain $\delta^*(\rho)$ than removing instructions. Given that the reduction of $\delta^*(\rho)$ implies the reduction of fitness expectation, LGP prefers larger programs than smaller programs, that is, bloat effect.
\end{remark}
Fig. \ref{fig:prog_size} verifies the bloat effect. Fig. \ref{fig:prog_size} shows that the program size of the LGP population consistently grows during evolution, although there is no bias in genetic operators (50\% for adding instructions and 50\% for removing).

\begin{figure}
    \centering
    \includegraphics[scale=0.75, viewport=30 20 420 320, clip=true]{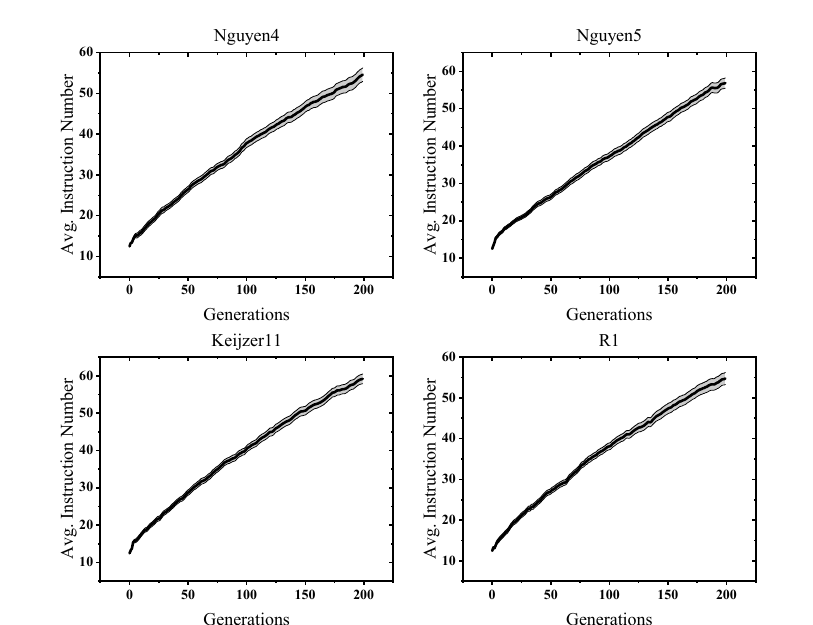}
    \caption{The average program size (the number of instructions) $\pm$std. of the population over 50 independent runs for the example problems.}
    \label{fig:prog_size}
\end{figure}

\textbf{Theorem \ref{theo:UB_Emin}} and the explanation of bloat effect back up the existing recommended initialization strategy and operator settings in LGP. First, we initialize the LGP population to a small program size so that we can have a small ${\mathbb{E}}[\delta^*(\rho)\ |\ \rho\in\mathcal{P}_m]$. 
Second, we encourage LGP to search from small to large program size so that we can reduce $\delta^*(\rho)$ with a higher probability. This explanation of the bloat effect also implies that it happens when we have a given editing distance $d_m$. However, a randomly sampled LGP population (e.g., initial population) might include a wide range of editing distances. Therefore, the average program size over an LGP population might decrease in the very early stage of evolution. This phenomenon is commonly observed in existing LGP studies \cite{huang_toward_2024}. In this beginning stage of evolution, LGP is disseminating relatively good individuals across the population so that most of the selected parents in breeding have a similar editing distance. Later, bloat happens.

\section{Expected Minimum Hitting Time of Optimal LGP Individuals}\label{sec:individual_move}
The hitting time of optimal LGP individuals indicates the running time with which LGP finds optimal individuals.
This section performs a case study on the expected minimum hitting time of LGP individuals based on a basic genetic operator, mutating random instructions without further constraints (also known as \textit{freemut} in \cite{Brameier2007}).
Specifically, we model the LGP evolution as reducing fitness supremums. 
LGP hits the optimal solutions within an expected minimum hitting time if LGP always moves with the upper bound on the constructive moving rate. A constructive move indicates a variation on LGP programs with a better instruction editing distance.

\begin{definition}
    \label{def:E_Detal_D}
Assume that an LGP individual $\rho$ with size $m$ has an initial $\delta^*(\rho)=d$. 
Let $\rho'_{(o)}$ be the offspring after one of the genetic operators $o\in\mathbb{O}$ is applied to $\rho$ for once, where $\mathbb{O}$ is the set of genetic operators. Let $\Delta \delta^*_{(\rho, o)} = \delta^*(\rho'_{(o)}) - \delta^*(\rho)$ be the difference of $\delta^*$ between $\rho'_{(o)}$ and $\rho$.
Then, the expectation of $\Delta \delta^*$ given a program $\rho$ is defined as:
    \begin{equation}
    \mathbb{E}[\Delta \delta^*_{(\rho)}]=\sum_{o\in \mathbb{O}} \Pr(o)\mathbb{E}[\Delta \delta^*_{(\rho, o)}] 
  = \sum_{o\in \mathbb{O}} \Pr(o) \left( \sum_{\Delta d\in \{\text{all possible }\Delta \delta^*_{(\rho,o)}\}} \Delta d \cdot \Pr(\Delta d) \right),  \nonumber
    \end{equation}
$$\Pr(\Delta d) = \frac{||\mathcal{P}^{\Delta \delta^*=\Delta d}_{(\rho, o)}||}{||\mathcal{P}_{(\rho, o)}||},$$
where $\mathcal{P}_{(\rho, o)}$ is the set of offspring programs that are generated by applying operator $o$ to $\rho$, and $\mathcal{P}^{\Delta \delta^*=\Delta d}_{(\rho, o)}$ is a subset of $\mathcal{P}_{(\rho, o)}$ in which $\Delta \delta^*_{(\rho, o)} = \Delta d$.

\end{definition}

We define the move from $\rho$ to $\rho'_{(o)}$ to be constructive if $\Delta \delta^*_{(\rho, o)} < 0$.
Then the expectation of constructive moves from $\rho$ to $\rho'$ is (\textbf{Definition \ref{def:E_Detal_D}}):
\begin{equation}
\label{eq:cons_Delta_d}
  \mathbb{E}[\Delta \delta^*_{(\rho)}\ |\ \Delta \delta^*_{(\rho)}<0]=
  \sum_{o\in \mathbb{O}} \Pr(o) \left( \sum_{\Delta d\in \{\text{all possible }\Delta \delta^*_{(\rho,o)}<0\}} \Delta d \cdot  \frac{||\mathcal{P}^{\Delta \delta^*=\Delta d}_{(\rho, o)}||}{||\mathcal{P}_{(\rho, o)}||} \right)  .
\end{equation}

We focus on a basic LGP genetic operator, \textit{freemut}, that produces offspring by adding or removing instructions. Freemut might add (remove) necessary and unnecessary instructions and introns into (from) program $\rho$. Both necessary and unnecessary instructions contribut to the program output, but necessary instructions constitute the target program while unnecessary instructions do not. Fig. \ref{fig:unnecessary_instr} shows an example of necessary and unnecessary instructions and introns for a specific program $\rho=\{\mathrm{R0}=x+1\}$ when $\rho^*=\{\mathrm{R0}=x+1;\mathrm{R0}=\mathrm{R0}+1\}$.

\begin{figure}
    \centering
    \includegraphics[scale=0.55, viewport=20 15 570 270, clip=true]{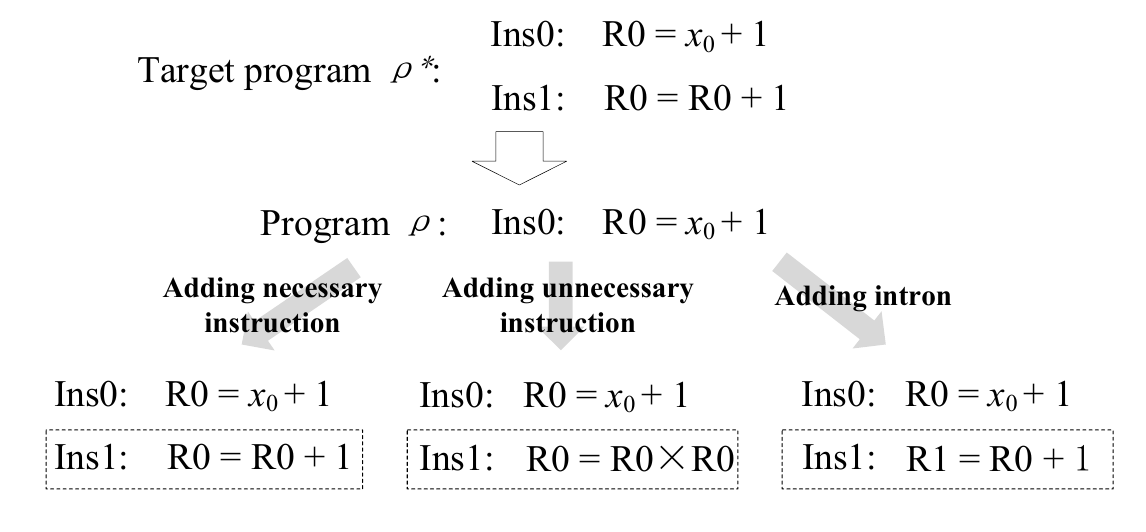}
    \caption{An example of necessary and unnecessary instructions and introns.}
    \label{fig:unnecessary_instr}
\end{figure}

\begin{lemma}\label{lemma:UB_con_add}
    Given an LGP individual $\rho$ whose $\delta^*(\rho)=a+r$ (i.e., accessing $\rho^*$ by adding $a$ necessary instructions and removing $r$ unnecessary instructions), we assume that an operator $o_{+u}$ that adds $u$ instructions into $\rho$ to produce offspring will add any number of unnecessary instructions ($j$) with the same probability. Then, the upper bound on the number of offspring that reduces $\delta^*(\rho)$ by $\Delta \delta^*_{(\rho,o_{+u})}=i$ is 
    \begin{align}
        & \overline{||\mathcal{P}^{\Delta \delta^*=i}_{(\rho, o_{+u})}||} = \frac{1}{\min\{ \lfloor (u-i)/2 \rfloor, \delta^*(\rho)-i\}+1} \sum_{j=0}^{\min\{ \lfloor (u-i)/2 \rfloor, \delta^*(\rho)-i\}} \nonumber \\
        &  \binom{\delta^*(\rho)}{i+j} \overline{\Lambda}(m+i+j,m+i+2j) \overline{\Omega}(m+i+2j, m+u), \nonumber
    \end{align}%
where $\overline{\Lambda}(\cdot)$ and $\overline{\Omega}(\cdot)$ are the upper bounds on $\Lambda$ and $\Omega$, respectively.
(For proof refer to Appendix \ref{prf:UB_con_add})
\end{lemma}

\begin{lemma}\label{lemma:UB_con_remove}
Given an LGP individual $\rho$ whose $\delta^*(\rho)=a+r$ (i.e., accessing $\rho^*$ by adding $a$ necessary instructions and removing $r$ unnecessary instructions), we assume that an operator $o_{-u}$ that removes $u$ instructions from $\rho$ to produce offspring will remove any number of unnecessary instructions ($j$) with the same probability. Then, the upper bound on the number of offspring that reduces $\delta^*(\rho)$ by $\Delta \delta^*_{(\rho, o_{-u})}=i$ is 
$$\overline{||\mathcal{P}^{\Delta \delta^*=i}_{(\rho, o_{-u})}||} = \binom{|\rho|}{u}.$$  
(For proof refer to Appendix \ref{prf:UB_con_remove})
\end{lemma}

\begin{theorem}\label{theo:cons_indi_u}
    When LGP only has two operators, adding and removing $u$ random instructions, each with 50\% operator rate, then the upper bound on constructive moving rate $|\mathbb{E}[\Delta \delta^*_{(\rho)}\ |\ \Delta \delta^*_{(\rho)}<0]|$ is: 
    \begin{equation}
        \overline{|\mathbb{E}[\Delta \delta^*_{(\rho)}\ |\ \Delta \delta^*_{(\rho)}<0]|} = 
         \frac{1}{2} \left( \frac{1}{\binom{|\rho|+u}{u} n^u} \sum_{i_1=1}^{I_1} i_1 \overline{||\mathcal{P}^{\Delta \delta^*=i_1}_{(\rho,o_{+u})}||} + \frac{I_2(1+I_2)}{2} \right), \nonumber
    \end{equation}
    where $I_1=\min\{\delta^*(\rho),u\}$, $I_2=\min\{\delta^*(\rho),|\rho|,\frac{\delta^*(\rho)+|\rho|-1}{2},u\}$.
\end{theorem}
\begin{proof}
Assume that $\delta^*(\rho)=a+r$ where $a$ is the number of necessary to-add instructions and $r$ is the number of necessary to-remove instructions. \\
\begin{enumerate}
    \item For adding $u$ instructions, the number of offspring $||\mathcal{P}_{(\rho, o_{+u})}||$ is $\binom{|\rho|+u}{u} n^u$.\\
    \item For removing $u$ instructions, the number of offspring $||\mathcal{P}_{(\rho, o_{-u})}||$ is $\binom{|\rho|}{u}$.\\
\end{enumerate} 
Thus, substituting Lemmas \ref{lemma:UB_con_add} and \ref{lemma:UB_con_remove} into Eq. (\ref{eq:cons_Delta_d}), i.e., replacing $||\mathcal{P}^{\Delta \delta^*=\Delta d}_{(\rho, o_{\pm u})}||$ by their upper bounds, we have the upper bound on constructive moving rate:
\begin{align}
  \label{eq:con_ub}
      \overline{|\mathbb{E}[\Delta \delta^*_{(\rho)}\ |\ \Delta \delta^*_{(\rho)}<0]|}&=\frac{1}{2}\left( \frac{1}{\binom{|\rho|+u}{u} n^u} \sum_{i_1=1}^{I_1} i_1 \overline{||\mathcal{P}^{\Delta \delta^*=i_1}_{(\rho,o_{+u})}||} + \frac{1}{\binom{|\rho|}{u}}\sum_{i_2=1}^{I_2} i_2 \overline{||\mathcal{P}^{\Delta \delta^*=i_2}_{(\rho, o_{-u})}||} \right)  \nonumber \\
      & =\frac{1}{2} \left( \frac{1}{\binom{|\rho|+u}{u} n^u} \sum_{i_1=1}^{I_1} i_1 \overline{||\mathcal{P}^{\Delta \delta^*=i_1}_{(\rho,o_{+u})}||} + \sum_{i_2=1}^{I_2} i_2 \right),
\end{align}%
where $I_1=\min\{\delta^*(\rho),u\}$, and $I_2=\min\{\delta^*(\rho),|\rho|,\frac{\delta^*(\rho)+|\rho|-1}{2},u\}$. The proofs of $I_1$ and $I_2$ refer to Appendices \ref{prf:UB_con_add} and \ref{prf:UB_con_remove}.
\end{proof}



\begin{remark}
    Since $\overline{||\mathcal{P}^{\Delta \delta^*=i}_{(\rho,o_{+u})}||}$ is an upper bound on the number of offspring with better $\delta^*$ by adding instructions, it is possible that $\sum_i^{I_1} \overline{||\mathcal{P}^{\Delta \delta^*=i}_{(\rho,o_{+u})}||} >\binom{|\rho|+u}{u} n^u$ (the same to $\overline{||\mathcal{P}^{\Delta \delta^*=i}_{(\rho,o_{-u})}||} > \binom{|\rho|}{u}$). To have a tight upper bound, we truncate the probability of $\frac{\overline{||\mathcal{P}^{\Delta \delta^*=i}_{(\rho,o_{+u})}||}}{\binom{|\rho|+u}{u} n^u}$ to 1 when $\overline{||\mathcal{P}^{\Delta \delta^*=i}_{(\rho,o_{+u})}||}>\binom{|\rho|+u}{u} n^u$ and scale down the two terms in Eq. (\ref{eq:con_ub}) by $\frac{1}{I_1}$ and $\frac{1}{I_2}$, respectively. Specifically, Theorem \ref{theo:cons_indi_u} can be re-written as:
    \begin{equation}
        \overline{|\mathbb{E}[\Delta \delta^*_{(\rho)}\ |\ \Delta \delta^*_{(\rho)}<0]|} =
        \frac{1}{2} \left(\sum_{i_1=1}^{I_1} \frac{i_1 \overline{||\mathcal{P}^{\Delta \delta^*=i}_{(\rho,o_{+u})}||}}{I_1 \max\{ \overline{||\mathcal{P}^{\Delta \delta^*=i}_{(\rho,o_{+u})}||}, \binom{|\rho|+u}{u} n^u\}}  + \frac{1+I_2}{2} \right). \nonumber
    \end{equation}
    
\end{remark}

\begin{corollary}\label{corol:minigeneration}
When LGP only has two operators, adding and removing $u$ random instructions, each with 50\% operator rate, then an LGP individual $\rho$ accesses the neighborhood of $\rho^*$ ($\{\rho|\delta^*(\rho) \leq \epsilon\}$) with a minimum number of generations $q$ so that
    $$ \underbrace{ D \circ \cdots \circ D(\delta^*(\rho))}_{q \times D(\cdot)} \leq \epsilon,$$
where $D(\delta^*(\rho))=\delta^*(\rho)-\overline{|\mathbb{E}[\Delta \delta^*_{(\rho)}\ |\ \Delta \delta^*_{(\rho)}<0]|}$, and $\epsilon$ is a small enough positive constant.
\end{corollary}


\begin{figure}[t]
    \centering
    \includegraphics[scale=0.6, viewport=10 320 600 600, clip=true]{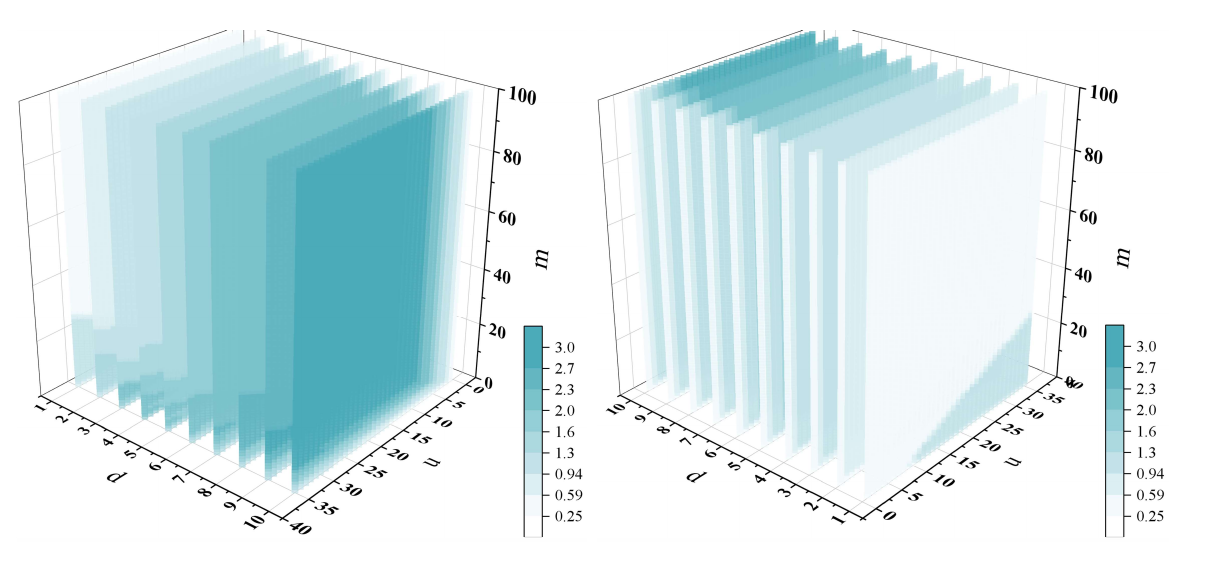}
    \caption{Each dot is a value of the upper bounds on the constructive moving rate of $\delta^*$ given variation step size $u$, editing distance $d$, and program size $m$ in Nguyen4. The two sub-figures show the values in two perspectives. The dark color indicates a large constructive moving rate upper bounds.}
    \label{fig:con_move_N_minigen}
\end{figure}
\begin{figure}[t]
    \centering
    \includegraphics[scale=0.6, viewport=10 20 600 300, clip=true]{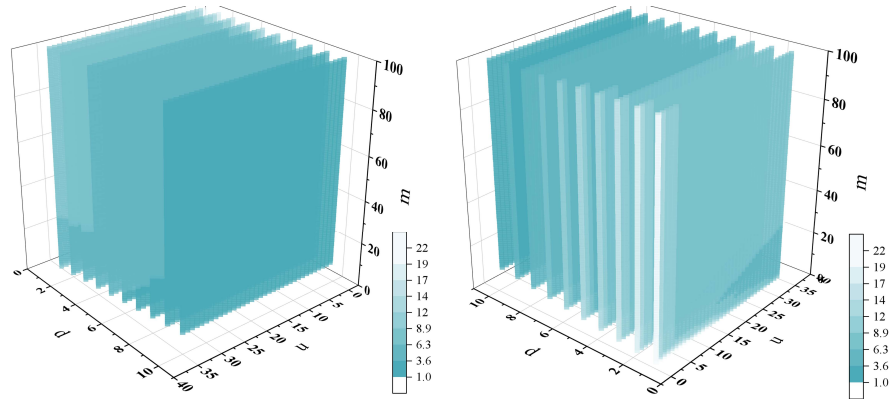}
    \caption{Each dot is a value of the minimum hitting time (number of generations) variation step size $u$, editing distance $d$, and program size $m$ in Nguyen4. The two sub-figures show the values in two perspectives. The dark color indicates a small minimum hitting time.}
    \label{fig:con_move_N_minigen2}
\end{figure}

Theorem \ref{theo:cons_indi_u} and Corollary \ref{corol:minigeneration} show the upper bounds on constructive moving rate ($\overline{|\mathbb{E}[\Delta \delta^*_{(\rho)}\ |\ \Delta \delta^*_{(\rho)}<0]|}$) and their corresponding minimum hitting time ($q$), respectively. To intuitively understand these formulas, Fig. \ref{fig:con_move_N_minigen} and \ref{fig:con_move_N_minigen2} show $\overline{|\mathbb{E}[\Delta \delta^*_{(\rho)}\ |\ \Delta \delta^*_{(\rho)}<0]|}$ and $q$ in Nguyen4 (smallest optimal program size $m^*=11$) by enumerating variation step size $u \in [1~..~35]$, instruction editing distance $d\in [1~..~10]$, program size $m\in[1~..~100]$, and set $\epsilon=1\times10^{-4}$.
We make the following three observations:
\begin{enumerate}
    \item The moving rate upper bounds and minimum hitting time is correlated to editing distance to the optimal solution ($d$). A large $d$ implies a larger moving rate and a smaller hitting time.
    \item The moving rate and the minimum hitting time with a small program size ($m\leq20$) are better than those with a large $m$. It implies that maintaining a necessarily small program size likely improves LGP search efficiency.
    \item When the variation step size $u$ increases from 1 to 10, we can see an increase in moving rate upper bounds and a reduction of minimum hitting time. However, when $u$ is larger than 10, most of the probability in $\overline{|\mathbb{E}[\Delta \delta^*_{(\rho)}\ |\ \Delta \delta^*_{(\rho)}<0]|}$ are truncated to 1. We cannot see an increase or reduction.
    This suggests that when we consider the variation step size of freemut in a range of small values (e.g., $u\in[1~..~10]$), a large step size $u$ improves LGP effectiveness.
\end{enumerate}

\subsection{Empirical Verification of Theorem \ref{theo:cons_indi_u}}
We verify \textbf{Theorem \ref{theo:cons_indi_u}} by comparing the convergence with the expected moving rates and minimum generations. 
Fig. \ref{fig:convergence} shows the convergence curves of LGP evolved by freemut with $u=1,3,5,7,9,12,15$ instructions. 
We can see that in the four example problems, LGP with large variation step sizes (e.g., $u=9,12,15$) converges to better fitness within fewer generations. Particularly, the blue curve ($u=9$) has the best convergence performance in three of the four problems. On the other hand, LGP with small variation step sizes (e.g., $u=1$) has the worst convergence speed among the compared settings. 
Although in Keijzer11, LGP with $u=9$ ends up with higher (worse) fitness values than $u=3,5$, the gap is marginal. 

For $u=12, 15$ where the constructive moving rate upper bounds are truncated,  the convergence performance is slightly worse than $u=9$. This implies that when the truncation in Theorem \ref{theo:cons_indi_u} happens, the performance of LGP might decrease with the increase of $u$.
The empirical results confirm that, given a limited range of variation step sizes for freemut, mutating multiple instructions in freemut is better than mutating only one instruction, which is implied by \textbf{Theorem \ref{theo:cons_indi_u}} and its numerical results (Fig. \ref{fig:con_move_N_minigen} and \ref{fig:con_move_N_minigen2}).



\begin{figure}
    \centering
    \includegraphics[scale=0.75, viewport=30 20 420 320, clip=true]{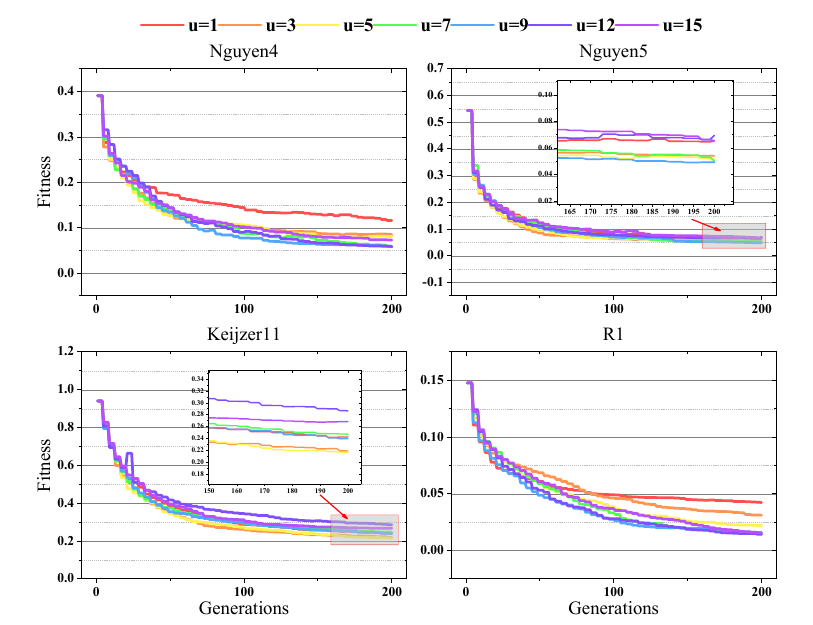}
    \caption{The convergence curves of LGP with adding and removing $u$ instructions}
    \label{fig:convergence}
\end{figure}

\section{Conclusions}\label{sec:conclude}

In this paper, we have modeled LGP behavior theoretically by fitness supremums which are linearly correlated to instruction editing distance. This finding bridges the gap between fitness values and LGP search spaces and dispenses the problem-specific designs in existing GP theoretical analysis (e.g., focusing on unrealistic problems ORDER and MAJORITY).
We show that relaxing fitness to its supremum based on instruction editing distance is a feasible way to understand LGP behavior.

There are four key insights:
\begin{enumerate}
    \item \textbf{Theorem \ref{theo:dtofit}} implies that a) reducing the semantic gap between input and target output semantics and b) reducing the semantic gap between different instructions accepting different input semantics lead to better fitness expectation.
    \item \textbf{Theorem \ref{theo:UB_Emin}} implies that the fitness expectation of LGP individuals increases (getting worse) with program size. This theoretically backs up the recommended initialization strategy of small GP programs.
    \item \textbf{Corollary \ref{corol:bloat_effect}} explains that bloat happens because it is more likely to reduce and maintain fitness by adding instructions than by removing them, given a program size and a certain editing distance to optimal solutions. 
    \item \textbf{Theorem \ref{theo:cons_indi_u}} suggests that LGP with the basic genetic operator (i.e., freemut) should use a large variation step size and keep program size small to achieve better performance.
\end{enumerate}

We believe that these results strengthen the theoretical foundation of GP and provide a possible perspective (i.e., fitness supremums based on instruction editing distance) for the promotion of further theoretical studies in genetic programming.

\small

\bibliographystyle{apalike}
\bibliography{ecjsample}

\section{Proof of Lemmas and Corollaries}\label{app:lemmaproof}

\subsection{Lemma \ref{lemma:theta}}\label{prf:theta}
Given a semantic space $\Psi$ and a set of instructions $\mathcal{I}$, we have $0 \leq \Delta_{(\mathcal{I}^*,\Psi)} \leq \Delta_{(\mathcal{I}^2,\Psi)}$.

\begin{proof}
    First, by \textbf{Definition \ref{def:theta_f}} ($\Delta_{(\mathcal{I}^*,\Psi)}$ is the minimum upper bound of $||\sigma(\mathbf{s}_1) - \sigma(\mathbf{s}_2)||-||\mathbf{s}_1 - \mathbf{s}_2||$), we must have 
    $$
    \Delta_{(\mathcal{I}^*,\Psi)} \geq ||\sigma(\mathbf{s}) - \sigma(\mathbf{s})|| - ||\mathbf{s} - \mathbf{s}|| = 0, \text{ where } \mathbf{s} \in \Psi.
    $$
    Second, by \textbf{Definition \ref{def:theta_f}} and \textbf{Definition \ref{def:theta_F}}, we have 
    \begin{equation}
        \Delta_{(\mathcal{I}^2,\Psi)} \geq \sup_{\substack{
        \mathbf{s}_1, \mathbf{s}_2 \in \Psi, \sigma \in \mathcal{I}^*}} (||\sigma(\mathbf{s}_1) - \sigma(\mathbf{s}_2)||-||\mathbf{s}_1 - \mathbf{s}_2||) = \Delta_{(\mathcal{I}^*,\Psi)}.
    \end{equation}
    Therefore, $0 \leq \Delta_{(\mathcal{I}^*,\Psi)} \leq \Delta_{(\mathcal{I}^2,\Psi)}$.
\end{proof}

\subsection{Lemma \ref{lemma:d}}\label{prf:d}
 Suppose the smallest program size that represents $\rho^*$ is $m^*$, and LGP only uses adding or removing one random instruction as genetic operators to produce offspring. Then an LGP solution with size $m$, $\rho_{(m)}$, has lower and upper bounds of $\delta^*$ to access $\rho^*$. Specifically,
$$\inf{\delta^*_{(m)}}=\max\{0, m^*-m\} \leq \delta^*_{(m)} \leq m^*+m = \sup{\delta^*_{(m)}}.$$
\begin{proof}
    \begin{enumerate}
        \item For the infimum,\\
        $\because$ when $m<m^*$, the ideal case for $\rho_{(m)}$ to reach $\rho^*$ is to add $m^*-m$ instructions. In this case, the instructions in $\rho_{(m)}$ are exactly the ones in $\rho^*_{(m^*)}$ and are arranged in the correct order. When $m\geq m^*$, $\rho_{(m)}$ already includes $\rho^*_{(m)}$ (i.e., $\delta^*=0$). 
        \\
        $\therefore$ $$\max\{0, m^*-m\} \leq \delta^*_{(m)}.$$
        \item For the supremum, \\
        $\because$ the worst case for $\rho_{(m)}$ is that all the instructions in $\rho_{(m)}$ are excluded from $\rho^*_{(m)}$.\\
        $\therefore$ $\rho_{(m)}$ has to first remove all its instructions (i.e., $m$ moves) and then add the right instructions in the right order (i.e., $m^*$ moves). \\
        $\therefore$ the minimum upper bound $\delta^*_{(m)}$ is $m^*+m$.
    \end{enumerate}
\end{proof}%

\subsection{Infimum and Supremum on Accessible $\rho^*$ Program Sizes $\phi$}
\label{prf:mp}
 Denote the program size of $\rho^*$ that are accessible by $\rho_{(m)}$ within $d$ steps as $\phi_{(m,d)}$, then $\phi_{(m,d)}$ has its infimum and supremum of 
    $\inf{\phi_{(m,d)}} = m+d-2\min\{\frac{d-m^*+m}{2}, d, m\} \leq \phi_{(m,d)} \leq m+d-2\max\{0, d-m^*\} = \sup{\phi_{(m,d)}}$. 
 \begin{proof}
Given an LGP program with size $m$, $\rho_{(m)}$, denote $r$ as the number of removing instructions, and $a$ is the number of adding instructions, so that $\delta^*(\rho)=d=r+a$. 
Then we have the following three relationships:
\begin{enumerate}
    \item the accessible program size of $\rho^*$ is equivalent to $\phi_{(m,d)}=m-r+a$, and $\phi_{(m,d)}$ must be greater than or equal to $m^* (m^*>0)$.
    \item the number of removing instructions $r$ must be less than or equal to the program size $m$.
    \item the number of adding instruction $a$ must be less than or equal to the smallest program size of $\rho^*$, i.e., $m^*$.
\end{enumerate}
The three relationships are formulated as follows.
\begin{align}
    \phi_{(m,d)} = m - r + a &\geq m^* > 0, \nonumber \\
    m \geq r &\geq 0, \nonumber \\
    m^* \geq a &\geq 0, \nonumber
\end{align}
$\therefore$ After substituting $a=d-r$ and $r=d-a$, respectively, we have:
\begin{equation}
     \max\{0, d-m^*\} \leq r \leq \min\{\frac{d-m^*+m}{2}, d, m\}, \nonumber
\end{equation}
\begin{equation}
    \max\{0, d-m, \frac{m^*+d-m}{2}\} \leq a \leq \min\{d, m^*\}. \nonumber
\end{equation} 
 
$\therefore$ For a given $m$ and $\delta^*(\rho)=d$, 
\begin{enumerate}
    \item $\phi_{(m,d)}$ has its infimum when $\rho$ removes the most $r=\min\{\frac{d-m^*+m}{2}, d, m\}$ instructions. We have $$\inf{\phi_{(m,d)}}=m+d-2\min\{\frac{d-m^*+m}{2}, d, m\}.$$
    \item $\phi_{(m,d)}$ has its supremum when $\rho$ removes the least $r=\max\{0, d-m^*\}$ instructions. We have $$\sup{\phi_{(m,d)}}=m+d-2\max\{0, d-m^*\}.$$
\end{enumerate}
\end{proof}

\subsection{Lemma \ref{lemma:UBLBomega}}\label{prf:UBLBomega}
 The lower and upper bounds of the neutral bloating factor is:
    $\sum_{i=\omega}^{\gamma-\gamma_{\text{out}}} \left( \frac{in}{\gamma}\right)^{m_2-m_1}<\Omega(m_1, m_2)<\left( \frac{m_2(\gamma-\gamma_{\text{out}}+\omega)(\gamma-\gamma_{\text{out}}-\omega+1)n}{2\gamma} \right)^{m_2 - m_1}$.
\begin{proof}
When an LGP program increases its size from $m_1$ to $m_2$ by neutral moves, the LGP program has to add $m_2-m_1$ introns. These $m_2-m_1$ introns allocate at different instruction positions over the program. Denote $\alpha_i (i=\omega, \omega+1, ..., \gamma-\gamma_{\text{out}})$ as the number of instruction positions that are added intron instructions, we have
$$\alpha_\omega+\alpha_{\omega + 1}+...+\alpha_{\gamma-\gamma_{\text{out}}}=m_2-m_1.$$
where $\omega$ denotes the minimum number of ineffective registers $\omega=\max\{1,\gamma-\gamma_{\text{out}}-m_1\}$. 

Different instruction positions imply different numbers of possible introns. For example, given that the LGP program has $\gamma_{\text{out}}$ output registers, there are $\frac{(\gamma-\gamma_{\text{out}})n}{\gamma}$ instructions being introns following the last exon. These $\frac{(\gamma-\gamma_{\text{out}})n}{\gamma}$ instructions are composed of the instructions whose destination registers are excluded from the output register set. For the basic LGP in which each instruction has up to two source registers, there are $\frac{(\gamma-\gamma_{\text{out}} - 1)n}{\gamma}$ introns following the second last exon but preceding to the last exon. \\
$\therefore$ For a certain number of possible introns $\frac{in}{\gamma}$, it repeats $a_i$ times in the combination, and we have $\left( \frac{in}{\gamma}\right)^{\alpha_i}$.

Suppose there are up to $A_i$ instruction positions for a certain number of introns $\frac{in}{\gamma}$ to repeat $\alpha_i$ times,\\
$\therefore$ there are $\binom{A_i}{\alpha_i}$ combinations for each certain number of introns, and the total number of the maximal instruction positions for adding introns is $$A_\omega+ A_{\omega + 1}+...+A_{\gamma - \gamma_{\text{out}}}=m_2.$$ 

Based on this rule, the possible numbers of introns at each instruction position are $\frac{\omega n}{\gamma}$, $\frac{(\omega+1)n}{\gamma}$, ..., $\frac{(\gamma-\gamma_{\text{out}})n}{\gamma}$ where $\omega = \max\{0, \gamma-\gamma_{\text{out}}- m_1\}$.\\ 
$\therefore$ The total number of combinations of introns is defined as $\Omega(m_1, m_2)$.

\begin{align}
    &\Omega(m_1, m_2) = \sum_{ 
    \begin{array}{c}
         \alpha_\omega+\alpha_{\omega + 1}+...+\alpha_{\gamma-\gamma_{\text{out}}}=m_2-m_1,  \\
         A_\omega+ A_{\omega + 1}+...+A_{\gamma-\gamma_{\text{out}}}=m_2 
    \end{array}
    } \nonumber \\
&\binom{A_\omega}{\alpha_\omega} \left(\frac{\omega n}{\gamma}\right)^{\alpha_\omega} \times \binom{A_{\omega+1}}{\alpha_{\omega+1}}\left(\frac{(\omega +1)n}{\gamma}\right)^{\alpha_{\omega+1}}\times... 
\times\binom{A_{\gamma-\gamma_{\text{out}}}}{\alpha_{\gamma-\gamma_{\text{out}}}}\left(\frac{(\gamma-\gamma_{\text{out}})n}{\gamma}\right)^{\alpha_{\gamma-\gamma_{\text{out}}}}. \nonumber 
\end{align}

\begin{enumerate}
    \item For the lower bound, we just simply sum up the items in which only one of the $\alpha_i (i=\omega, \omega+1,..., \gamma-\gamma_{\text{out}})$ is non-zero, and have $\Omega(m_1, m_2)>\sum_{i=\omega}^{\gamma-\gamma_{\text{out}}} \binom{A_i}{\alpha_i}\left( \frac{in}{\gamma}\right)^{m_2-m_1}$. Other $\binom{A_j}{0} \left( \frac{jn}{\gamma}\right)^{\alpha_j} = 1 (j\neq i, \alpha_j=0, \binom{A_j}{0}=1)$.\\
    $\because$
    $1 \leq \binom{A_i}{\alpha_i}$,\\
    $\therefore$
    $$\sum_{i=\omega}^{\gamma-\gamma_{\text{out}}} \left( \frac{in}{\gamma}\right)^{m_2-m_1}<\Omega(m_1, m_2).$$
    \item For the upper bound, \\
    $\because$ 
    $$\binom{A_\omega}{\alpha_\omega}\binom{A_{\omega+1}}{\alpha_{\omega+1}}...\binom{A_{\gamma-\gamma_{\text{out}}}}{\alpha_{\gamma-\gamma_{\text{out}}}} =: \frac{C_1}{C_2 \times C_3} ,$$
    where $C_1=A_\omega! A_{\omega+1}! ... A_{\gamma-\gamma_{\text{out}}}!$,\\
    $C_2=(A_\omega - \alpha_\omega)!(A_{\omega+1} - \alpha_{\omega+1})!...(A_{\gamma-\gamma_{\text{out}}} - \alpha_{\gamma-\gamma_{\text{out}}})!$,\\
    $C_3=\alpha_\omega!\alpha_{\omega+1}! ... \alpha_{\gamma-\gamma_{\text{out}}}!$.
    
    $\because$ $\frac{A_i!}{(A_i-\alpha_i)!}=(A_i-\alpha_i+1)(A_i-\alpha_i+2)...A_i$ $(i=\omega,\omega+1,...,\gamma-\gamma_{\text{out}})$,\\
    $\therefore$ $\frac{A_i!}{(A_i-\alpha_i)!}$ has $\alpha_i$ multiplying items, and each item must be less than or equal to $m_2$ since $A_i\leq m_2$.\\
    $\therefore$ $\frac{\mathcal{A}}{\mathcal{B}}$ has $\alpha_\omega+\alpha_{\omega+1}+...+\alpha_{\gamma-\gamma_{\text{out}}}=m_2-m_1$ multiplying items which are all $\leq m_2$.\\
    $\therefore$ $\frac{C_1}{C_2} < m_2^{m_2-m_1}(m_2 - m_1)!$ where $(m_2 - m_1)!$ has $m_2 - m_1$ items that are all larger than $1$.\\
    $\therefore$ 
    \begin{align}
        & \frac{C_1}{C_2C_3} \left(\frac{\omega n}{\gamma}\right)^{\alpha_\omega}\left(\frac{(\omega +1)n}{\gamma}\right)^{\alpha_{\omega+1}}...\left(\frac{(\gamma-\gamma_{\text{out}})n}{\gamma}\right)^{\alpha_{\gamma-\gamma_{\text{out}}}} \nonumber \\
        & < \frac{m_2^{m_2-m_1}(m_2-m_1)!}{\mathcal{C}} \left(\frac{\omega n}{\gamma}\right)^{\alpha_\omega}...\left(\frac{(\gamma-\gamma_{\text{out}})n}{\gamma}\right)^{\alpha_{\gamma-\gamma_{\text{out}}}} \nonumber \\
        & = m_2^{m_2-m_1}\binom{m_2 - m_1}{\alpha_\omega,\alpha_{\omega+1},...,\alpha_{\gamma-\gamma_{\text{out}}}} \left(\frac{\omega n}{\gamma}\right)^{\alpha_\omega}...\left(\frac{(\gamma-\gamma_{\text{out}})n}{\gamma}\right)^{\alpha_{\gamma-\gamma_{\text{out}}}}. \nonumber
    \end{align}
    $\therefore$ (\textbf{Multinomial Theorem})
        \begin{align}
        &\Omega(m_1, m_2)< \nonumber \\
        &m_2^{m_2-m_1} \sum_{\alpha_\omega+\alpha_{\omega + 1}+...+\alpha_{\gamma-\gamma_{\text{out}}}=m_2-m_1}\binom{m_2-m_1}{\alpha_\omega,\alpha_{\omega+1},...,\alpha_{\gamma-\gamma_{\text{out}}}} \nonumber \\
        &\left(\frac{\omega n}{\gamma}\right)^{\alpha_\omega}\left(\frac{(\omega +1)n}{\gamma}\right)^{\alpha_{\omega+1}}...\left(\frac{(\gamma-\gamma_{\text{out}})n}{\gamma}\right)^{\alpha_{\gamma-\gamma_{\text{out}}}} \nonumber \\
        & = m_2^{m_2-m_1} \left( \frac{\omega n}{\gamma} + \frac{(\omega +1)n}{\gamma} + ... + \frac{(\gamma-\gamma_{\text{out}})n}{\gamma} \right)^{m_2-m_1} \nonumber \\
        & = \left( \frac{m_2(\gamma-\gamma_{\text{out}}+\omega)(\gamma-\gamma_{\text{out}}-\omega+1)n}{2\gamma} \right)^{m_2 - m_1}. \nonumber
    \end{align}
\end{enumerate}
\end{proof}

\subsection{Lemma \ref{lemma:UBLBlambda}}
\label{prf:UBLBlambda}
 The lower and upper bounds of the non-neutral bloating factor is:
    $\sum_{i=\gamma_{\text{out}}}^{\lambda} \left( \frac{in}{\gamma}\right)^{m_2-m_1}<\Lambda(m_1, m_2)<\left( \frac{m_2(\gamma_{\text{out}}+\lambda)(\lambda -\gamma_{\text{out}} +1)n}{2\gamma} \right)^{m_2 - m_1}$
    where $\lambda=\min\{\gamma,\gamma_{\text{out}}+m_1\}$.
\begin{proof}
    The proof is similar to \textbf{Lemma \ref{prf:UBLBomega}}, where $\omega$ and $\gamma-\gamma_{\text{out}}$ are replaced by $\gamma_{\text{out}}$ and $\lambda$, respectively.
\end{proof}

\subsection{Number of Unique Solutions After Removing $r$ Instructions}
\label{prf:numRemoving}
Given $||\mathcal{P}_{m}^0||$ unique programs, the number of unique programs after removing $r$ instructions from each of the $\rho\in\mathcal{P}_{m}^0$ programs is:
    $$ \frac{||\mathcal{P}_{m}^0||\binom{m}{r}}{\eta_r} ,$$
    where $\eta_r$ is the normalization factor of duplicated programs after removing. $\eta_r$ has its own lower and upper bounds.
    $$1 < \eta_r< \binom{m}{r}\left( \frac{(\gamma-\gamma_{\text{out}}) n}{\gamma}\right)^{r}.$$
\begin{proof}
1) For the lower bound of $\eta_r$:\\
The minimum number of duplicated programs after removing is 1, which indicates no duplication. $\therefore 1<\eta_r$.\\
2) For the upper bound of $\eta_r$:\\
    $\because$ the $||\mathcal{P}_{m}^0||$ unique programs have the same $m^*$ instructions and $m-m^*$ different intron instructions.\\
    $\therefore$ the generated programs might be duplicated if we remove the intron instructions.\\
    $\therefore$ we have the largest number of duplicated programs when all the $r$ removing instructions are intron instructions.\\
    $\because$ different positions have different numbers of possible intron instructions.\\
    $\therefore$ we have the largest number of duplicated programs when all the $r$ instruction positions have the largest number of possible intron instructions.\\
    $\therefore$ the largest number of duplicated programs is $\binom{m-m^*}{r}\left( \frac{(\gamma-\gamma_{\text{out}}) n}{\gamma}\right)^{r}$.\\
    $\because$ $m-m^* < m$, $\therefore$ $\binom{m-m^*}{r}<\binom{m}{r}$. \\
    $\therefore$ $\eta_r < \binom{m}{r}\left( \frac{(\gamma-\gamma_{\text{out}}) n}{\gamma}\right)^{r}$.
\end{proof}

\subsection{Number of Unique Solutions After Adding $a$ Instructions}
\label{prf:numAdding}
 Given $N$ unique programs with size of $m$, the number of unique programs after adding $a$ instructions into each of the $N$ program is:
    $$\frac{N\binom{m+a}{a}n^{a}}{\eta_a},$$
    where $\eta_a$ is the normalization factor of duplicated programs after adding.
    $$1 < \eta_a < N\binom{m+a}{a}.$$
\begin{proof}
    1) For the lower bound of $\eta_a$:\\
    The minimum number of duplicated programs after adding is 1, which indicates no duplication. $\therefore 1<\eta_a$.\\
    2) For the upper bound of $\eta_a$:\\
    $\because$ a unique program at most generates $\binom{m+a}{a}$ duplicated offspring when the newly added instructions are the same as the instructions in the program, and all the instructions in the program are the same.\\
    $\because$ a generated program at most has $N-1$ duplicated ones that are generated from other $N-1$ unique programs.\\ 
    $\therefore$ $\eta_a < N\binom{m+a}{a}$.
\end{proof}

\subsection{Lemma \ref{lemma:UB_con_add}}\label{prf:UB_con_add}
Given an LGP individual $\rho$ whose $\delta^*(\rho)=a+r$ (i.e., accessing $\rho^*$ by adding $a$ necessary instructions and removing $r$ unnecessary instructions), we assume that an operator $o_{+u}$ that adds $u$ instructions into $\rho$ to produce offspring will add any number of unnecessary instructions ($j$) with the same probability. Then, the upper bound of the number of offspring that reduces $\delta^*(\rho)$ by $\Delta \delta^*_{(\rho,o_{+u})}=i$ is \par
    \begin{align}
    &\overline{||\mathcal{P}^{\Delta \delta^*=i}_{(\rho, o_{+u})}||} = \frac{1}{\min\{ \lfloor (u-i)/2 \rfloor, \delta^*(\rho)-i\}+1} \sum_{j=0}^{\min\{ \lfloor (u-i)/2 \rfloor, \delta^*(\rho)-i\}} \nonumber \\
    &  \binom{\delta^*(\rho)}{i+j} \overline{\Lambda}(m+i+j,m+i+2j) \overline{\Omega}(m+i+2j, m+u), \nonumber
\end{align}
where $\overline{\Lambda}(\cdot)$ and $\overline{\Omega}(\cdot)$ are the upper bounds of $\Lambda$ and $\Omega$, respectively.
\begin{proof}
    Based on Appendix \ref{prf:mp}, we have $\max\{0,\delta^*(\rho)-m\} \leq a \leq \delta^*(\rho) $ by reducing $m^*$ and relaxing the lower and upper bounds of $a$.\par
    $\because$ the constructive moves by the operator $o_{+u}$ consist of three types of moves: 1) adding instructions necessary instructions, 2) adding unnecessary exons, and 3) adding introns. \par
    $\therefore$ suppose that $o_{+u}$ constructively reduces $\delta^*(\rho)$ by $i$, then LGP adds $i+j$ necessary instructions, $j$ unnecessary exons, and $u-i-2j$ introns, then the number of neighbors of an individual move is\par
    $$\frac{\binom{a}{i+j} \Lambda(m+i+j,m+i+2j) \Omega(m+i+2j, m+u) }{\eta_a} ,$$     
    where $\eta_a$ is the number of the duplicated new programs after adding $u$ instructions.  \par
    $\because$ $i+j \leq a$ and $i\geq 1$,$j,u-i-2j\geq 0$ (Appendix \ref{prf:mp}).\par
    $\therefore$ $1\leq i \leq \min\{a,u\}$ and $0 \leq j \leq \min\{\lfloor (u-i)/2 \rfloor, a-i\}$.\par
    $\therefore$ we enumerate $j$ to estimate the expected total number of neighbors that reduce $\delta^*(\rho)$ by $i$ is (based on \textbf{Lemmas \ref{lemma:UBLBomega} and \ref{lemma:UBLBlambda}} and Appendix \ref{prf:numAdding}):\par
    \begin{align}
    &|| \mathcal{P}^{\Delta \delta^*=i}_{(\rho, o_{+u})} || =  \sum_{j=0}^{\min\{ \lfloor (u-i)/2 \rfloor, a-i\}} \Pr(j) \binom{a}{i+j} \Lambda(m+i+j,m+i+2j) \Omega(m+i+2j, m+u) \nonumber \\
        &\leq \sum_{j=0}^{\min\{ \lfloor (u-i)/2 \rfloor, \delta^*(\rho)-i\}} 
        \Pr(j) \binom{\delta^*(\rho)}{i+j} \overline{\Lambda}(m+i+j,m+i+2j) \overline{\Omega}(m+i+2j, m+u), \nonumber
    \end{align}    
    where $1\leq i \leq \min\{\delta^*(\rho),u\}$.\par
   $\because$ we assume that different values of $j$ share the same probability, $\therefore$ $\Pr(j)=\frac{1}{\min\{ \lfloor (u-i)/2 \rfloor, d-i\}+1}$.
    
\end{proof}

\subsection{Lemma \ref{lemma:UB_con_remove}}\label{prf:UB_con_remove}
Given an LGP individual $\rho$ whose $\delta^*(\rho)=a+r$ (i.e., accessing $\rho^*$ by adding $a$ necessary instructions and removing $r$ unnecessary instructions), we assume that an operator $o_{-u}$ that removes $u$ instructions from $\rho$ to produce offspring will remove any number of unnecessary instructions ($j$) with the same probability. Then, the upper bound of the number of offspring that reduces $\delta^*(\rho)$ by $\Delta \delta^*_{(\rho, o_{-u})}=i$ is 
  $$\overline{||\mathcal{P}^{\Delta \delta^*=i}_{(\rho, o_{-u})}||} = \binom{|\rho|}{u}. $$
  
\begin{proof}
Based on Appendix \ref{prf:mp}, we have $0 \leq r \leq \min\{\delta^*(\rho),|\rho|,\frac{\delta^*(\rho)+|\rho|-1}{2}\}$ by reducing $m^*$ as 1 and relaxing the lower bound.\par
    $\because$ the constructive moves by removing $u$ instructions consist of three types of moves: 1) removing unnecessary exons, 2) removing necessary instructions, and 3) removing introns. \par
    $\therefore$ suppose that $o_{-u}$ constructively reduces $\delta^*(\rho)$ by $i$, LGP removes $i+j$ unnecessary exons, $j$ necessary instructions, and $u-i-2j$ introns, where $i+j\leq r$, $i\geq 1$, $j, u-i-2j\geq 0$.\par
    $\therefore$ $1\leq i \leq \min\{r,u\}$, $0\leq j \leq \min\{\lfloor (u-i)/2 \rfloor, r-i\}$.\par
    $\therefore$ Suppose there are $N_{nec}$ necessary instructions and $N_{intron}$ introns in the program. Given $i$ and $j$, the number of neighbors of an LGP individual is $\binom{r}{i+j}\binom{N_{nec}}{j}\binom{N_{intron}}{u-i-2j}/\eta_r$, where $\eta_r$ is the number of duplicated new programs by removing instructions. \par
    $\because$ $r+N_{nec}+N_{intron}=|\rho|$.\par
    $\therefore$ we enumerate $j$ to estimate the expected total number of neighbors that reduce $\delta^*(\rho)$ by $i$ is (based on Appendix \ref{prf:numRemoving}):
    \begin{align}
        &||\mathcal{P}^{\Delta \delta^*=i}_{(\rho, o_{-u})}|| = 
        \sum_{j=0}^{\min\{ \lfloor (u-i)/2 \rfloor, r-i\}} \frac{P(j)}{\eta_r} \binom{r}{i+j} \binom{N_{nec}}{j}\binom{|\rho|-r-N_{nec}}{u-i-2j} \nonumber \\
        & \leq \sum_{j=0}^{\min\{ \lfloor (u-i)/2 \rfloor, r-i\}} P(j) \binom{r}{i+j} \binom{|\rho|-r}{u-i-j} \nonumber \\
& \leq \sum_{j=0}^{\min\{ \lfloor (u-i)/2 \rfloor, r-i\}} P(j) \binom{|\rho|}{u}, \nonumber
\end{align}
where $1\leq i \leq \min\{\delta^*(\rho),|\rho|,\frac{\delta^*(\rho)+|\rho|-1}{2},u\}$.\\
    $\because$ We assume that different values of $j$ share the same probability, $\therefore$ $\Pr(j)=\frac{1}{\min\{ \lfloor (u-i)/2 \rfloor, r-i\}+1}$. 
    $\therefore ||\mathcal{P}^{\Delta \delta^*=i}_{(\rho, o_{-u})}||\leq  \binom{|\rho|}{u}$.
\end{proof}

\end{document}